\documentclass[preprint,11pt]{elsarticle}

\biboptions{numbers,sort&compress}

\usepackage{amsmath,amsfonts,amssymb,amsthm}
\usepackage{mathtools}
\mathtoolsset{showonlyrefs}
\usepackage{graphicx}
\usepackage{xcolor}
\definecolor{addedblue}{RGB}{0,70,170}
\definecolor{correctionred}{RGB}{190,0,0}

\newcommand{\corr}[1]{\textcolor{correctionred}{#1}}

\usepackage{caption,placeins}
\usepackage{subcaption}
\usepackage{hyperref}
\hypersetup{colorlinks=true,linkcolor=blue,citecolor=blue,urlcolor=blue}
\usepackage[capitalise,nameinlink]{cleveref}

\newcommand{\sym}{\operatorname{sym}}
\renewcommand{\skew}{\operatorname{skew}}

\newtheorem{lemma}{Lemma}
\newtheorem{remark}{Remark}
\newtheorem{theorem}{Theorem}
\newtheorem{corollary}{Corollary}

\journal{\corr{Journal of Mathematical Analysis and Applications}}

\begin{document}

\begin{frontmatter}

\title{Reliable Error Estimation for PINNs: Lower and Upper A Posteriori Bounds}

\author[aff1,aff2]{Ismail Huseynov}
\author[aff3]{Arzu Ahmadova}
\author[aff4]{Agamirza Bashirov}

\affiliation[aff1]{organization={Physikalisch-Technische Bundesanstalt (PTB)}, country={Germany}}
\affiliation[aff2]{organization={Technical University of Berlin}, country={Germany}}
\affiliation[aff3]{organization={Weierstrass Institute for Applied Analysis and Stochastics}, city={Berlin}, country={Germany}}
\affiliation[aff4]{organization={Eastern Mediterranean University}, country={Turkey}}

\begin{abstract}
Physics-informed neural networks (PINNs) combine machine learning with physical laws to solve differential equations. While existing results provide rigorous \emph{a posteriori} upper bounds for PINN prediction errors, complete certification also requires complementary lower information in order to obtain computable two-sided error enclosures. In this paper, we derive computable \emph{a posteriori} lower bounds for PINN errors in ordinary differential equations on suitable certified state-space domains under a localized strong monotonicity condition. We combine these estimates with complementary localized upper bounds under a one-sided Lipschitz condition, which is weaker than the global Lipschitz assumption used in previous work and can yield sharper upper error bands. The resulting bounds depend only on the neural-network approximation, the ODE residual, and local monotonicity and growth constants, and therefore do not require access to the exact solution. For linear time-invariant and time-varying systems, we further derive explicit formulas in terms of the minimal and maximal eigenvalues of the symmetric part of the system matrix. We also discuss the distinction between soft and hard enforcement of initial conditions in PINNs and explain why exact enforcement can make the scalar lower certificate uninformative. To recover nontrivial lower information in the linear setting, we use a signed-residual finite-probe certificate based on coordinate unit vectors. We also formulate a certificate-informed training strategy in which the propagated upper certificate is used as an auxiliary regularizer, while lower certificates remain post-training diagnostics. Altogether, the proposed framework provides rigorous and practically computable error certificates for PINN approximations of ODEs, while making explicit the domains and model classes for which the assumptions can be verified.
\end{abstract}


\end{frontmatter}

\section{Introduction}

The certification of machine learning methods has emerged as a critical challenge in scientific computing, particularly for physics-informed neural networks (PINNs) where reliability guarantees are essential for deployment in safety-critical applications. Significant progress has been made in deriving a posteriori upper bounds for PINN errors \cite{hillebrecht2022certified}. While upper bounds estimate the proximity of the exact and approximate solutions from above, lower bounds provide complementary information on the minimal separation between them. Complete certification therefore requires both upper and lower error enclosures to provide tight bounds on prediction quality. However, the complementary problem of establishing rigorous lower bounds remains largely unexplored.

\textbf{Literature review.}
Neural network methods for solving ordinary and partial differential equations (ODEs and PDEs), especially PINNs, have become very popular in recent years, as discussed in \cite{cuomo2022sciml}. PINNs incorporate physical constraints directly into the learning objective through differential operators, enabling them to leverage neural network expressivity while respecting physical consistency. This hybrid approach facilitates accurate predictions even with limited or noisy data by constructionally enforcing physical feasibility \cite{raissi2019physics}.

Certified error estimation for PINNs builds upon foundational developments in scientific machine learning \cite{raissi2019physics,karniadakis2021physicsinformed}. Existing approaches cover statistical uncertainty quantification and a deep confidence framework \cite{cortes2019deep}, dual-weighted residual methods \cite{minakowski2021error}, and goal-oriented error estimation \cite{vandermeer2022goal}, though these typically provide either statistical rather than rigorous bounds or require substantial training data.

Statistical error models quantify uncertainty only in a statistical sense: they rely on sufficiently rich and representative training data to construct confidence intervals for the prediction error and therefore do not, in general, provide deterministic rigorous certificates for a given trained PINN. In contrast, certified \emph{a posteriori} error estimation aims at computable worst-case bounds for the prediction error of a specific network realization, without requiring knowledge of the exact solution and, in the sense of \cite{hillebrecht2022certified}, even for inputs not used during training. In physics-informed problems, we often lack a large labeled dataset of true errors, so direct application of deep confidence is limited. Drawing from certification frameworks in reduced-order modeling \cite{hesthaven2016certified} and classical stability theory for ODEs \cite{hairer2008solving}, this work establishes rigorous \emph{a posteriori error bounds} that remain valid for unseen data without requiring \emph{a priori knowledge of the true solution}. Theoretical analyses by \cite{deryck2022genericbound} further advance PINN reliability by deriving rigorous, dimension-robust \emph{a priori approximation error bounds} for both solution and operator-learning networks, while providing constructive and efficient guidelines that explicitly relate network architecture to achievable accuracy.

While the certification of PINNs for ODEs has received increasing attention, extending such guarantees to PDEs poses substantially greater theoretical and computational challenges. In the PDE setting, certification requires accounting for spatial derivatives, boundary and initial conditions, and stability properties of the underlying operator. Recent studies have begun to address these issues by deriving rigorous a posteriori error bounds and verification frameworks for PDE-defined PINNs; see \cite{haugen2025trustworthy,hillebrecht2025rigorous} and references therein. Hillebrecht and Unger \cite{hillebrecht2025rigorous,hillebrecht2025stokes} established one of the first rigorous certification schemes for PINNs governed by classical PDEs such as the heat, transport, and Navier--Stokes equations. Their approach derives provable upper bounds on the prediction error using PDE stability estimates and residual norms, enabling verification without access to the exact solution. In parallel, \cite{eiras2024efficient} introduced an efficient certification method based on bounding the global residual, providing computable correctness conditions analogous to solver tolerances. Furthermore, \cite{fanaskov2024neuralfunctional} proposed a functional-type a posteriori framework that integrates classical majorant estimators into PINN training, yielding guaranteed error upper bounds with high computational efficiency. Guo and Haghighat \cite{guo2020elasticity} extended these ideas to elasticity problems, employing the constitutive relation error to obtain both global and goal-oriented error estimates. However, establishing corresponding \textit{a posteriori lower bounds}, which would provide guaranteed measures of the minimal attainable error, remains largely unexplored in the current literature.

\noindent\textbf{Practical applicability and scope.}
The theory developed below is designed for ODE models and operating regimes in which the vector field satisfies suitable local one-sided growth conditions on a certified state-space domain. Representative examples include early-growth epidemic models and their multigroup linearizations \cite{chowell2016growthreview,diekmann2010ngm,hoang2025infectious}, low-incidence SIS-type reductions, and autocatalytic or other positive-feedback chemical kinetics in batch and flow reactors \cite{schuster2019autocatalysis,hanopolskyi2020autocatalysis,kriukov2024autocatalytic}. From a practical viewpoint, the relevant issue is therefore local rather than global: one verifies the required sign conditions for the symmetric part of the Jacobian on a certified domain containing both the exact and approximate trajectories. This motivates the localized framework adopted in the present paper.

\noindent\textbf{Initial conditions and informativeness of lower bounds.}
The informativeness of a lower certificate depends crucially on how initial conditions are imposed in the PINN. In standard PINNs, the initial condition is typically enforced \emph{softly} through a data-misfit term in the loss, rather than exactly by construction \cite{raissi2019physics}. \emph{Hard-constraint} formulations instead embed the condition into the network architecture or trial ansatz so that it is satisfied exactly by construction \cite{lu2021hardconstraints,hao2024stability}. In the latter case, the scalar lower bound can become trivial after \emph{nonnegativity} post-processing because the anchoring error vanishes. We therefore do not claim that the scalar estimator alone can certify a positive lower error in this regime. Instead, when the model is linear, we use the signed residual itself through a finite set of coordinate probes which gives a full-interval lower diagnostic. This distinction is practically relevant because recent work on stiff time-dependent PINNs shows that exact hard enforcement of initial and boundary conditions can materially affect robustness and training efficiency \cite{hao2024stability}.

\noindent\textbf{Organization and contributions.}
Accordingly, the main contributions of this manuscript are:
\begin{itemize}
    \item computable localized \emph{a posteriori} lower bounds for PINN errors on certified state-space domains under \textit{a strong monotonicity condition},
    \item complementary localized \emph{a posteriori} upper bounds under \textit{a one-sided Lipschitz condition}, extending earlier certification results based on Lipschitz assumptions and yielding potentially sharper upper certificates on certified domains \cite{hillebrecht2022certified,hillebrecht2025rigorous},
    \item explicit formulas for \textit{linear time-invariant and time-varying systems} in terms of the extremal eigenvalues of the symmetric part of the system matrix,
    \item \textit{a certified-containment criterion} that makes precise when the local assumptions remain valid and when certification must terminate,
    \item a signed-residual finite-probe lower certificate for linear inhomogeneous systems, implemented with coordinate unit vectors and remaining informative when the scalar lower bound vanishes,
    \item a certificate-informed training formulation based on the propagated upper certificate and its auxiliary scalar ODE representation.
\end{itemize}

These two-sided certified error bands provide mathematically rigorous enclosures of the PINN approximation error. A narrow band indicates tight certification, whereas a wide band reflects either conservative local constants or limited physical fidelity of the trained network. The band width therefore serves as a quantitative indicator of certificate sharpness and practical reliability.

The paper is organized as follows. Section~\ref{Sec 2} introduces the problem formulation, the localized certification setting on state-space domains, and the flow-map notation used throughout the paper. Section~\ref{Sec 3} presents the main \emph{a posteriori} error estimates, including localized lower and upper bounds for nonlinear systems, sharpened formulas for linear time-invariant and time-varying systems, and the RK4-based numerical post-processing used to evaluate the certification terms. Section~\ref{Sec 4} explains the connection with physics-informed neural-network training. Section~\ref{Sec 5} develops the certificate-informed training formulation. Section~\ref{Sec 6} reports numerical examples illustrating the theory. Section~\ref{Sec 7} summarizes the main conclusions and limitations.

\textbf{Notation.} We conclude this introductory section by introducing standard mathematical notation used throughout the paper. The time interval is denoted by $\mathbb{T}=[0,T]$ with $T>0$. For vectors in $\mathbb{R}^n$, $\|\cdot\|$ denotes the Euclidean norm and $\cdot$ the standard inner product. For a matrix $A$, $A^\top$ denotes its transpose. For a finite set $A$, we write $|A|$ for its cardinality. Unless stated otherwise, all norms in the theoretical developments are understood in the Euclidean sense, although several arguments extend to other equivalent norms on finite-dimensional spaces.

\section{Problem Description}\label{Sec 2}

We consider the initial value problem on the time interval \(\mathbb{T}=[0,T]\):
\begin{equation}\label{eq:ode}
\begin{cases}
\dot{x}(t)=f(t,x(t)), \quad t\in\mathbb{T},\\
x(0)=x_0,
\end{cases}
\end{equation}
where \(f:\mathbb{T}\times\mathbb{R}^n\to\mathbb{R}^n\) is continuous and
\(x_0\in\mathbb{R}^n\).

A standard sufficient condition for local existence and uniqueness of solutions to
\eqref{eq:ode} is continuity in \(t\) together with local Lipschitz continuity of \(f\)
with respect to \(x\); see, for example, the Picard--Lindel\"of theorem \cite{arnold1978}.
In the present paper, however, we work on a \emph{certified} convex state-space domain
\(D\subset\mathbb{R}^n\). All structural assumptions are therefore imposed only on
\(\mathbb{T}\times D\), and the resulting certificates are valid on every time interval on
which both the exact trajectory and the PINN approximation remain in \(D\).

For clarity, we first formulate the assumptions with \emph{constant} certified
coefficients on \(D\). Thus, for all \(u,v\in D\) and all \(t\in\mathbb{T}\), we assume
\begin{align}
(f(t,u)-f(t,v))\cdot(u-v) &\ge \ell_D \|u-v\|^2,
\label{eq:local-lower-osl}\\
(f(t,u)-f(t,v))\cdot(u-v) &\le \mu_D \|u-v\|^2,
\label{eq:local-upper-osl}\\
\|f(t,u)-f(t,v)\| &\le L_D \|u-v\|.
\label{eq:local-lip}
\end{align}
Here,
\begin{itemize}
    \item \(\ell_D\in\mathbb{R}\) is a \emph{local lower one-sided Lipschitz coefficient},
    \item \(\mu_D\in\mathbb{R}\) is a \emph{local upper one-sided Lipschitz constant},
    \item \(L_D>0\) is a \emph{local Lipschitz constant}.
\end{itemize}

The terminology is important: the inequality \eqref{eq:local-lower-osl} is a lower
one-sided Lipschitz estimate in general, and it becomes a \emph{strong monotonicity}
condition precisely in the case \(\ell_D>0\). This is the regime in which the lower
\emph{a posteriori} error bound is informative. By contrast, \eqref{eq:local-upper-osl}
is the relevant assumption for upper error bounds; if \(\mu_D<0\), then the dynamics are
contractive on \(D\).

The role of the three coefficients is therefore different. The local Lipschitz bound
\eqref{eq:local-lip}, together with continuity in \(t\), provides the standard well-posedness assumption on \(D\). The lower one-sided bound \eqref{eq:local-lower-osl} is used for lower error certificates, whereas the upper one-sided bound \eqref{eq:local-upper-osl} is used for upper error certificates. Since \eqref{eq:local-lip} implies \eqref{eq:local-upper-osl} with the same constant, the upper one-sided Lipschitz assumption is weaker than the Lipschitz assumption used in earlier upper-bound results such as \cite{hillebrecht2022certified,hillebrecht2025rigorous}.

\begin{remark}[Signs and admissible constants]
\label{rem:signs-admissibility}
The one-sided coefficients \(\ell_D\) and \(\mu_D\) are real numbers; they are not required to be positive. Positivity of \(\ell_D\) is needed only when one wants to call the lower estimate a strong-monotonicity estimate. Similarly, \(\mu_D<0\) is allowed and means that the dynamics are contractive on the certified domain. By contrast, a Lipschitz constant satisfies \(L_D\ge0\). In the linear autonomous case, every \(\ell\le m(A)\) is an admissible lower one-sided coefficient, every \(\mu\ge M(A)\) is an admissible upper one-sided coefficient, and every \(L\ge\|A\|_2\) is an admissible Euclidean Lipschitz constant.
\end{remark}

If \(f\) is continuously differentiable with respect to \(x\), then these certified
coefficients can be chosen from the symmetric part of the Jacobian. Writing
\begin{equation}
    \sym(J_x f)(t,\xi):=\frac{J_x f(t,\xi)+J_x f(t,\xi)^\top}{2},
\end{equation}
one may define
\begin{align}
\ell_D &:= \inf_{(t,\xi)\in \mathbb{T}\times D}
\lambda_{\min}\!\bigl(\sym(J_x f)(t,\xi)\bigr),\\
\mu_D &:= \sup_{(t,\xi)\in \mathbb{T}\times D}
\lambda_{\max}\!\bigl(\sym(J_x f)(t,\xi)\bigr),\\
L_D &:= \sup_{(t,\xi)\in \mathbb{T}\times D}\|J_x f(t,\xi)\|_2.
\end{align}
In the scalar case \(n=1\), this reduces to
\begin{align}
\ell_D &= \inf_{(t,\xi)\in \mathbb{T}\times D}\partial_x f(t,\xi),\\
\mu_D &= \sup_{(t,\xi)\in \mathbb{T}\times D}\partial_x f(t,\xi),\\
L_D &= \sup_{(t,\xi)\in \mathbb{T}\times D}|\partial_x f(t,\xi)|.
\end{align}

For each initial value \(\xi\in D\) such that the corresponding solution exists on
\(\mathbb{T}\) and remains in \(D\), we denote that solution by \(t\mapsto\varphi(t,\xi)\).
This defines the associated local flow map
\begin{equation}\label{eq:flow-map}
\varphi:\mathbb{T}\times D \to \mathbb{R}^n,
\qquad
(t,\xi)\mapsto \varphi(t,\xi),
\end{equation}
on the subset of \(\mathbb{T}\times D\) where the trajectory stays in \(D\).
In particular, for the prescribed initial value \(x_0\), the exact trajectory is
\begin{equation}\label{optimal}
x(t)=\varphi(t,x_0),\quad t\in\mathbb{T}.
\end{equation}
In general, the flow map is not available in closed form and must therefore be approximated numerically. In this work, we use a physics-informed neural network. More precisely, we introduce a parameterized candidate map
\begin{equation}\label{eq:ml-candidate}
\hat{\varphi}_\theta:\mathbb{T}\times D\to\mathbb{R}^n,
\qquad
(t,\xi)\mapsto \hat{\varphi}_\theta(t,\xi),
\end{equation}
where \(\theta\in\mathbb{R}^k\) denotes the trainable parameter vector of the network. The architecture determines the approximation class, while \(\theta\) collects the weights and biases. Throughout the paper, we assume that \(\hat{\varphi}_\theta\) is sufficiently smooth with respect to \(t\) so that the differential-equation residual is well defined.

For the prescribed initial value \(x_0\), the corresponding PINN trajectory is the slice
\begin{equation}
\hat{x}_\theta(t):=\hat{\varphi}_\theta(t,x_0), \quad t\in\mathbb{T}.
\end{equation}
The training procedure returns a parameter vector, denoted by \(\theta^\star\), and we write
\begin{equation}
\hat{x}(t):=\hat{x}_{\theta^\star}(t)=\hat{\varphi}_{\theta^\star}(t,x_0).
\end{equation}
Here, \(\theta^\star\) simply denotes the trained parameter vector produced by the chosen optimization procedure.

The approximation error is then
\begin{equation}\label{eq:error}
e(t):=x(t)-\hat{x}(t)
=\varphi(t,x_0)-\hat{\varphi}_{\theta^\star}(t,x_0),
\quad t\in\mathbb{T}.
\end{equation}

Our goal is to estimate this error \emph{a posteriori}, using only computable quantities derived from the trained PINN. More precisely, the certificates below depend on
\begin{itemize}
    \item the trained trajectory \(\hat{x}\),
    \item a computable majorant of the residual,
    \item certified coefficients \(\ell_D\), \(\mu_D\), and \(L_D\) on the domain \(D\),
    \item and the initial mismatch \(\|e(0)\|=\|x_0-\hat{x}(0)\|\).
\end{itemize}
In \textit{soft-constrained} PINNs, this initial mismatch is controlled by the initial-condition part of the training loss and is directly computable from the trained network. In \textit{hard-constrained} PINNs, one has \(\hat{x}(0)=x_0\) by construction and therefore \(e(0)=0\).

\medskip
\noindent\textbf{Main Problem.}
Given a trained PINN approximation \(\hat{\varphi}_{\theta^\star}\), or equivalently the associated trajectory \(\hat{x}(t)=\hat{\varphi}_{\theta^\star}(t,x_0)\), derive computable and rigorous \emph{a posteriori} certificates for the error \eqref{eq:error} on \(\mathbb{T}\), in particular lower and upper bounds that can be evaluated at low computational cost without using the exact solution.

\begin{remark}[Time-dependent coefficients]
One can likewise consider time-dependent local coefficient functions \(\ell_D,\mu_D,L_D\in C(\mathbb{T})\). In that case, the constant exponential factors are replaced by the corresponding integrating factors
\begin{align}
\Lambda_D(t) &:=\int_0^t \ell_D(s)\,ds,\\
\mathcal M_D(t) &:=\int_0^t \mu_D(s)\,ds,\\
\mathcal L_D(t) &:=\int_0^t L_D(s)\,ds.
\end{align}
Accordingly, expressions of the form \(e^{\ell_D t}\), \(e^{\mu_D t}\), and \(e^{L_D t}\) are replaced by \(e^{\Lambda_D(t)}\), \(e^{\mathcal M_D(t)}\), and \(e^{\mathcal L_D(t)}\), respectively.
\end{remark}

\section{A Posteriori Error Estimation: Localized Lower and Upper Bounds}
\label{Sec 3}

In this section, we derive computable a posteriori error bounds for a trained PINN approximation. We treat the problem in two stages. First, we establish localized lower and upper error bounds for nonlinear systems on a certified state-space domain under one-sided growth assumptions. Second, we specialize these results to linear systems, where the relevant constants admit explicit spectral representations in terms of the symmetric part of the system matrix. Finally, we record the RK4-based numerical post-processing used to evaluate the certification integrals.

\subsection{Nonlinear case: localized lower and upper bounds}
\label{sec:nonlinear-localized-bounds}

Assume that the neural network has already been trained and that, for the prescribed initial value \(x_0\in\mathbb{R}^n\), we have obtained the PINN approximation \eqref{optimal}. Throughout this subsection, we work on a convex certified state-space domain \(D\subset\mathbb{R}^n\) and assume that both the exact solution \(x\) and the PINN approximation \(\hat{x}\) belong to \(C^1(\mathbb{T};D)\). Equivalently, the estimates below are valid on every time interval on which both trajectories remain in \(D\).

Since the candidate map $\hat{\varphi}_\theta$ is assumed to be sufficiently smooth, we define the residual by
\begin{equation}\label{eq:residual}
\mathcal{R}_{\hat{\varphi}}(t)
:=
\dot{\hat{x}}(t)-f(t,\hat{x}(t)),
\quad t\in\mathbb{T}.
\end{equation}
The time derivative can be computed efficiently by \emph{automatic differentiation} \cite{Rall1981}. Our aim in this section is to derive \textit{localized} \emph{a posteriori} lower and upper bounds for the prediction error
\begin{equation}\label{eq:error-section3}
e(t):=x(t)-\hat{x}(t),
\quad t\in\mathbb{T}.
\end{equation}

\begin{theorem}[Localized lower bound]
\label{thm:localized-lower}
Let $D\subset\mathbb{R}^n$ be convex, and let $x,\hat{x}\in C^1(\mathbb{T};D)$, where $x$ solves \eqref{eq:ode}. Assume that there exists a constant $\ell_D\in\mathbb{R}$ such that
\begin{equation}\label{eq:lower-osl-thm}
\bigl(f(t,u)-f(t,v)\bigr)\cdot(u-v)\ge \ell_D\|u-v\|^2
\end{equation}
for all $u,v\in D$ and all $t\in\mathbb{T}$. Let $\delta:\mathbb{T}\to\mathbb{R}_+$ be continuous and satisfy
\begin{equation}\label{eq:residual-upper}
\|\mathcal{R}_{\hat{\varphi}}(t)\|\le \delta(t),
\quad t\in\mathbb{T}.
\end{equation}
Then, for every $t\in\mathbb{T}$, the prediction error satisfies
\begin{equation}\label{eq:error-lower-localized}
\|e(t)\|
\ge
e^{\ell_D t}\|e(0)\|
-
\int_0^t e^{\ell_D(t-s)}\delta(s)\,ds.
\end{equation}
In particular, if $\ell_D>0$, then \eqref{eq:error-lower-localized} yields a localized strong-monotonicity lower bound.
\end{theorem}

\begin{proof}
Let
\begin{equation}\label{eq:app-error-def}
e(t):=x(t)-\hat{x}(t),
\quad t\in\mathbb{T}.
\end{equation}
Since \(x\) solves \eqref{eq:ode} and
\begin{equation}\label{eq:app-residual-def}
\mathcal{R}_{\hat{\varphi}}(t)
=
\dot{\hat{x}}(t)-f(t,\hat{x}(t)),
\end{equation}
the error satisfies
\begin{equation}\label{eq:error-dynamics-lower-app}
\dot e(t)
=
f(t,x(t))-f(t,\hat{x}(t))
-
\mathcal{R}_{\hat{\varphi}}(t).
\end{equation}
Because \(e\in C^1(\mathbb{T};\mathbb{R}^n)\), the map \(t\mapsto \|e(t)\|\) is absolutely continuous. Hence, for almost every \(t\in\mathbb{T}\) such that \(e(t)\neq 0\),
\begin{equation}\label{eq:app-norm-derivative-lower}
\frac{d}{dt}\|e(t)\|
=
\frac{e(t)\cdot \dot e(t)}{\|e(t)\|}.
\end{equation}
Combining \eqref{eq:error-dynamics-lower-app} and \eqref{eq:app-norm-derivative-lower}, we obtain
\begin{equation}\label{eq:app-lower-split}
\frac{d}{dt}\|e(t)\|
=
\frac{e(t)\cdot\bigl(f(t,x(t))-f(t,\hat{x}(t))\bigr)}{\|e(t)\|}
-
\frac{e(t)\cdot \mathcal{R}_{\hat{\varphi}}(t)}{\|e(t)\|}.
\end{equation}
Since \(x(t),\hat{x}(t)\in D\) and \eqref{eq:lower-osl-thm} holds,
\begin{equation}\label{eq:app-lower-monotonicity}
e(t)\cdot\bigl(f(t,x(t))-f(t,\hat{x}(t))\bigr)
\ge
\ell_D\|e(t)\|^2.
\end{equation}
Moreover, by Cauchy--Schwarz and \eqref{eq:residual-upper},
\begin{equation}\label{eq:app-lower-residual}
-e(t)\cdot \mathcal{R}_{\hat{\varphi}}(t)
\ge
-\|e(t)\|\,\|\mathcal{R}_{\hat{\varphi}}(t)\|
\ge
-\|e(t)\|\,\delta(t).
\end{equation}
Substituting \eqref{eq:app-lower-monotonicity} and \eqref{eq:app-lower-residual} into \eqref{eq:app-lower-split} yields
\begin{equation}\label{eq:app-lower-diff-ineq}
\frac{d}{dt}\|e(t)\|
\ge
\ell_D\|e(t)\|-\delta(t)
\quad \text{for a.e. } t\in\mathbb{T}.
\end{equation}
Multiplying \eqref{eq:app-lower-diff-ineq} by \(e^{-\ell_D t}\) gives
\begin{equation}\label{eq:app-lower-integrating-factor}
\frac{d}{dt}\Bigl(e^{-\ell_D t}\|e(t)\|\Bigr)
\ge
-e^{-\ell_D t}\delta(t)
\quad \text{for a.e. } t\in\mathbb{T}.
\end{equation}
Integrating from \(0\) to \(t\) yields
\begin{equation}\label{eq:app-lower-integrated}
e^{-\ell_D t}\|e(t)\|-\|e(0)\|
\ge
-\int_0^t e^{-\ell_D s}\delta(s)\,ds.
\end{equation}
Multiplying by \(e^{\ell_D t}\), we obtain
\begin{equation}
\|e(t)\|
\ge
e^{\ell_D t}\|e(0)\|
-
\int_0^t e^{\ell_D(t-s)}\delta(s)\,ds,
\end{equation}
which is exactly \eqref{eq:error-lower-localized}.
\end{proof}

\begin{remark}\label{rem:lower-nonnegative}
Since $\|e(t)\|\ge 0$, the practical lower certificate is obtained by taking the nonnegative truncation of the right-hand side of \eqref{eq:error-lower-localized}.
\end{remark}

\begin{remark}[When is the lower certificate informative?]
\label{rem:lower-informative}
Define the raw localized lower estimator by
\begin{equation}\label{eq:lower-raw-estimator}
\underline E_{\mathrm{loc,raw}}(t)
:=
e^{\ell_D t}\|e(0)\|
-
\int_0^t e^{\ell_D(t-s)}\delta(s)\,ds.
\end{equation}
Then the practical lower certificate from \Cref{rem:lower-nonnegative} is nontrivial exactly when
\begin{equation}
\underline E_{\mathrm{loc,raw}}(t)>0.
\end{equation}
In particular, if \(e(0)=0\) and only an upper residual majorant is available, then the lower certificate in \Cref{thm:localized-lower} may become trivial after nonnegativity post-processing. This motivates the signed-residual finite-probe certificate for linear systems in \Cref{thm:signed-residual-probe}; no positive lower bound can be obtained from scalar residual-majorant information alone in this regime.
\end{remark}

\subsection{Nontrivial lower information beyond scalar residual majorants}
\label{sec:lower-repair}

The scalar lower certificate in \Cref{thm:localized-lower} is intentionally based only on the residual norm majorant. This is useful because it is computationally intensive and does not require directional information, but it also explains the main failure mode. If
\begin{equation}
 e^{\ell_Dt}\|e(0)\|\le \int_0^t e^{\ell_D(t-s)}\delta(s)\,ds,
\end{equation}
then the nonnegative truncation of the scalar estimator is zero. This should not be interpreted as a numerical defect. It is an information limitation: from the quantities \(\|e(0)\|\), \(\ell_D\), and a scalar upper bound \(\|\mathcal R_{\hat\varphi}\|\le \delta\) alone, the residual may point in a direction that cancels the transported initial error. A positive lower certificate in this regime therefore requires additional computable information.

\begin{theorem}[Signed-residual finite-probe lower certificate for linear systems]
\label{thm:signed-residual-probe}

Consider the linear inhomogeneous autonomous system
\begin{equation}
 x'(t)=Ax(t)+g(t),\qquad x(0)=x_0,
\end{equation}
and let \(\hat x\in C^1([0,T];\mathbb R^n)\) be an approximation with residual
\begin{equation}
 R(t)=\hat x'(t)-A\hat x(t)-g(t).
\end{equation}
Then the error \(e(t)=x(t)-\hat x(t)\) satisfies the exact identity
\begin{equation}
 e(t)=e^{At}e(0)-\int_0^t e^{A(t-s)}R(s)\,ds.
\end{equation}
Consequently, for every finite set \(P\subset\mathbb R^n\) of unit vectors,
\begin{equation}
 \|e(t)\|\ge
 \max_{p\in P}\left|
 p^\top e^{At}e(0)-\int_0^t p^\top e^{A(t-s)}R(s)\,ds
 \right|.
\label{eq:probe-lower}
\end{equation}
\end{theorem}

\begin{proof}
Subtracting the approximate equation from the exact equation gives
\begin{equation}
 e'(t)=Ae(t)-R(t).
\end{equation}
The variation-of-constants formula gives
\begin{equation}
 e(t)=e^{At}e(0)-\int_0^t e^{A(t-s)}R(s)\,ds.
\end{equation}
For any unit vector \(p\), the Cauchy--Schwarz inequality gives \(|p^\top e(t)|\le \|e(t)\|\). Taking the maximum over the finite probe set \(P\) proves \eqref{eq:probe-lower}. 
\end{proof}

\begin{theorem}[Localized upper bound under upper one-sided Lipschitz continuity]
\label{thm:localized-upper}
Let $D\subset\mathbb{R}^n$ be convex, and let $x,\hat{x}\in C^1(\mathbb{T};D)$, where $x$ solves \eqref{eq:ode}. Assume that there exists a constant $\mu_D\in\mathbb{R}$ such that
\begin{equation}\label{eq:upper-osl-thm}
\bigl(f(t,u)-f(t,v)\bigr)\cdot(u-v)\le \mu_D\|u-v\|^2
\end{equation}
for all $u,v\in D$ and all $t\in\mathbb{T}$. Let $\delta:\mathbb{T}\to\mathbb{R}_+$ be continuous and satisfy \eqref{eq:residual-upper}. Then, for every $t\in\mathbb{T}$, the prediction error satisfies
\begin{equation}\label{eq:error-upper-localized}
\|e(t)\|
\le
e^{\mu_D t}\|e(0)\|
+
\int_0^t e^{\mu_D(t-s)}\delta(s)\,ds.
\end{equation}
\end{theorem}

\begin{proof}
Let
\begin{equation}
e(t):=x(t)-\hat{x}(t),
\quad t\in\mathbb{T}.
\end{equation}
Then
\begin{equation}\label{eq:error-dynamics-upper-app}
\dot e(t)
=
f(t,x(t))-f(t,\hat{x}(t))
-
\mathcal{R}_{\hat{\varphi}}(t).
\end{equation}
For almost every \(t\in\mathbb{T}\) such that \(e(t)\neq 0\),
\begin{equation}\label{eq:app-norm-derivative-upper}
\frac{d}{dt}\|e(t)\|
=
\frac{e(t)\cdot \dot e(t)}{\|e(t)\|}.
\end{equation}
Using \eqref{eq:error-dynamics-upper-app}, we obtain
\begin{equation}\label{eq:app-upper-split}
\frac{d}{dt}\|e(t)\|
=
\frac{e(t)\cdot\bigl(f(t,x(t))-f(t,\hat{x}(t))\bigr)}{\|e(t)\|}
-
\frac{e(t)\cdot \mathcal{R}_{\hat{\varphi}}(t)}{\|e(t)\|}.
\end{equation}
By \eqref{eq:upper-osl-thm},
\begin{equation}\label{eq:app-upper-monotonicity}
e(t)\cdot\bigl(f(t,x(t))-f(t,\hat{x}(t))\bigr)
\le
\mu_D\|e(t)\|^2.
\end{equation}
Also, by Cauchy--Schwarz and \eqref{eq:residual-upper},
\begin{equation}\label{eq:app-upper-residual}
-e(t)\cdot \mathcal{R}_{\hat{\varphi}}(t)
\le
\|e(t)\|\,\|\mathcal{R}_{\hat{\varphi}}(t)\|
\le
\|e(t)\|\,\delta(t).
\end{equation}
Therefore,
\begin{equation}\label{eq:app-upper-diff-ineq}
\frac{d}{dt}\|e(t)\|
\le
\mu_D\|e(t)\|+\delta(t)
\quad \text{for a.e. } t\in\mathbb{T}.
\end{equation}
Multiplying \eqref{eq:app-upper-diff-ineq} by \(e^{-\mu_D t}\) yields
\begin{equation}\label{eq:app-upper-integrating-factor}
\frac{d}{dt}\Bigl(e^{-\mu_D t}\|e(t)\|\Bigr)
\le
e^{-\mu_D t}\delta(t)
\qquad \text{for a.e. } t\in\mathbb{T}.
\end{equation}
Integrating from \(0\) to \(t\) gives
\begin{equation}\label{eq:app-upper-integrated}
e^{-\mu_D t}\|e(t)\|-\|e(0)\|
\le
\int_0^t e^{-\mu_D s}\delta(s)\,ds.
\end{equation}
Multiplying by \(e^{\mu_D t}\), we obtain
\begin{equation}
\|e(t)\|
\le
e^{\mu_D t}\|e(0)\|
+
\int_0^t e^{\mu_D(t-s)}\delta(s)\,ds,
\end{equation}
which is exactly \eqref{eq:error-upper-localized}.
\end{proof}

\begin{theorem}[Upper one-sided Lipschitz continuity is weaker than Lipschitz continuity]\label{thm:osl-weaker-than-lip}
Assume that there exists a constant \(L_D\ge 0\) such that
\begin{equation}\label{eq:local-lip-prop}
\|f(t,u)-f(t,v)\|\le L_D\|u-v\|
\end{equation}
for all \(u,v\in D\) and all \(t\in\mathbb{T}\). Then
\begin{equation}\label{eq:lip-implies-osl}
\bigl(f(t,u)-f(t,v)\bigr)\cdot(u-v)\le L_D\|u-v\|^2
\end{equation}
for all \(u,v\in D\) and all \(t\in\mathbb{T}\).
Hence every local Lipschitz constant is also a valid upper one-sided Lipschitz constant. The converse is false in general.
\end{theorem}

\begin{proof}
Assume that \eqref{eq:local-lip-prop} holds. Then, for all \(u,v\in D\) and all \(t\in\mathbb{T}\), the Cauchy--Schwarz inequality yields
\begin{equation}\label{eq:app-cs-osl}
\bigl(f(t,u)-f(t,v)\bigr)\cdot(u-v)
\le
\|f(t,u)-f(t,v)\|\,\|u-v\|.
\end{equation}
Using \eqref{eq:local-lip-prop}, we obtain
\begin{equation}\label{eq:app-lip-implies-osl}
\bigl(f(t,u)-f(t,v)\bigr)\cdot(u-v)
\le
L_D\|u-v\|^2.
\end{equation}
Hence every local Lipschitz constant is also a valid upper one-sided Lipschitz constant.

To prove that the converse fails in general, consider the scalar function
\begin{equation}\label{eq:app-cubic}
f(x)=-x^3,
\qquad x\in\mathbb{R}.
\end{equation}
Then, for all \(u,v\in\mathbb{R}\),
\begin{equation}\label{eq:app-cubic-osl}
(f(u)-f(v))(u-v)
=
-(u-v)^2(u^2+uv+v^2)
\le 0.
\end{equation}
Thus \(f\) is globally upper one-sided Lipschitz with constant \(\mu=0\). However, \(f\) is not globally Lipschitz on \(\mathbb{R}\), because
\begin{equation}\label{eq:app-cubic-derivative}
f'(x)=-3x^2
\end{equation}
is unbounded as \(|x|\to\infty\). Therefore upper one-sided Lipschitz continuity is strictly weaker than Lipschitz continuity.
\end{proof}

Consequently, the upper one-sided Lipschitz assumption in \Cref{thm:localized-upper} is strictly weaker than the Lipschitz-continuity assumption used in earlier certification results based on Lipschitz bounds.

\begin{remark}\label{rem:sharper-than-lipschitz}
Whenever a local Lipschitz constant $L_D$ is available, \Cref{thm:localized-upper} applies with $\mu_D=L_D$ by \Cref{thm:osl-weaker-than-lip}. If, however, one can estimate a smaller upper one-sided Lipschitz constant $\mu_D<L_D$, then \eqref{eq:error-upper-localized} yields a sharper upper error certificate.
\end{remark}

\begin{remark}[Hilbert-space PDE analogue]
\label{rem:pde-analogue}
The arguments above are formulated for ordinary differential equations in \(\mathbb{R}^n\), but their natural continuation is to evolution equations in a Hilbert space \(X\). In that setting, the Euclidean inner product is replaced by the Hilbert-space inner product, and coercivity or one-sided Lipschitz estimates are imposed on the underlying operator in the relevant function space.

Consider, formally, a linear evolution equation
\begin{equation}\label{eq:pde-linear-formal}
\partial_t u(t)=\mathcal A u(t)+g(t)
\end{equation}
in \(X\), together with an approximate solution \(\hat u\). The corresponding residual is then
\begin{equation}\label{eq:pde-residual-formal}
\mathcal R_{\hat u}(t)
:=
\partial_t \hat u(t)-\mathcal A \hat u(t)-g(t).
\end{equation}
The analogues of the linear ODE constants are the lower and upper logarithmic bounds
\begin{align}
m_X(\mathcal A)
&:=
\inf_{\substack{v\in D(\mathcal A)\\ \|v\|_X=1}}
\Re\langle \mathcal A v,v\rangle_X,
\\
M_X(\mathcal A)
&:=
\sup_{\substack{v\in D(\mathcal A)\\ \|v\|_X=1}}
\Re\langle \mathcal A v,v\rangle_X,
\end{align}
whenever these quantities are finite. Under suitable semigroup well-posedness and regularity assumptions, the same energy argument formally yields lower and upper error bounds in the norm of \(X\), with exponential kernels generated by \(m_X(\mathcal A)\) and \(M_X(\mathcal A)\).

Thus, the present ODE theory is consistent with the semigroup-based PDE certification viewpoint developed in \cite{hillebrecht2025rigorous,hillebrecht2025stokes}. Since the present paper is devoted to ODEs, we include this observation only as a conceptual remark rather than as a full PDE theorem.
\end{remark}

\subsection{Linear case: sharpened bounds}
\label{sec:linear-systems}

For linear systems, the localized nonlinear framework specializes naturally to constants derived from the symmetric part of the system matrix. In particular, the nonlinear quantities \(\ell_D\), \(\mu_D\), and \(L_D\) are replaced, in the linear case, by the lower symmetric-part bound, the upper symmetric-part bound, and the Euclidean operator norm, respectively.

We first consider the linear time-invariant system
\begin{equation}\label{eq:matrix}
f(t,x)=Ax,
\quad
A\in\mathbb{R}^{n\times n}.
\end{equation}
Its symmetric part is
\begin{equation}\label{eq:symmetric-part}
\sym A:=\frac{A+A^\top}{2}.
\end{equation}
We further define
\begin{equation}\label{eq:mA-MA}
m(A):=\lambda_{\min}(\sym A),
\quad
M(A):=\lambda_{\max}(\sym A).
\end{equation}

\begin{lemma}[Sharp linear one-sided constants]
\label{lem:linear-one-sided}
Let \(f(x)=Ax\). Then, for all \(u,v\in\mathbb{R}^n\),
\begin{equation}\label{eq:linear-quad-form}
(f(u)-f(v))\cdot(u-v)
=
(u-v)^\top (\sym A)(u-v).
\end{equation}
Consequently,
\begin{equation}\label{eq:linear-spectral-sandwich}
m(A)\|z\|^2
\le
z^\top (\sym A) z
\le
M(A)\|z\|^2,
\quad z\in\mathbb{R}^n.
\end{equation}
Hence \(m(A)\) is the sharp lower one-sided Lipschitz constant and \(M(A)\) is the sharp upper one-sided Lipschitz constant for the linear map \(x\mapsto Ax\). In particular, the strong monotonicity condition holds if and only if \(m(A)>0\).
\end{lemma}

\begin{proof}
Let \(f(x)=Ax\) and set
\begin{equation}\label{eq:app-linear-w}
w:=u-v.
\end{equation}
Then
\begin{equation}\label{eq:app-linear-start}
(f(u)-f(v))\cdot(u-v)
=
w^\top A w.
\end{equation}
Decompose \(A\) into its symmetric and skew-symmetric parts:
\begin{equation}\label{eq:app-linear-decomp}
A=\sym A+\skew A,
\end{equation}
where
\begin{equation}\label{eq:app-sym-skew}
\sym A:=\frac{A+A^\top}{2},
\qquad
\skew A:=\frac{A-A^\top}{2}.
\end{equation}
Since \(\skew A^\top=-\skew A\), we have
\begin{equation}\label{eq:app-skew-zero}
w^\top(\skew A)w=0
\qquad \text{for all } w\in\mathbb{R}^n.
\end{equation}
Therefore,
\begin{equation}\label{eq:app-linear-sym-only}
(f(u)-f(v))\cdot(u-v)
=
w^\top (\sym A) w.
\end{equation}
Since \(\sym A\) is symmetric, the Rayleigh--Ritz characterization yields
\begin{equation}\label{eq:app-rayleigh}
\lambda_{\min}(\sym A)\|w\|^2
\le
w^\top(\sym A)w
\le
\lambda_{\max}(\sym A)\|w\|^2.
\end{equation}
Using the definitions of \(m(A)\) and \(M(A)\), this becomes
\begin{equation}\label{eq:app-linear-sandwich}
m(A)\|w\|^2
\le
w^\top(\sym A)w
\le
M(A)\|w\|^2.
\end{equation}
This proves \eqref{eq:linear-quad-form} and \eqref{eq:linear-spectral-sandwich}.

It follows immediately that every constant \(\ell\le m(A)\) is a valid lower one-sided Lipschitz constant, and every constant \(\mu\ge M(A)\) is a valid upper one-sided Lipschitz constant. Hence \(m(A)\) is the sharp lower one-sided Lipschitz constant and \(M(A)\) is the sharp upper one-sided Lipschitz constant. In particular, the strong monotonicity condition holds if and only if \(m(A)>0\).
\end{proof}

\begin{lemma}[Matrix-exponential bounds]
\label{lem:exp-bounds-eig}
Let \(A\in\mathbb{R}^{n\times n}\), and let \(m(A)\) and \(M(A)\) be defined by \eqref{eq:mA-MA}. Then, for all \(t\ge 0\) and all \(z\in\mathbb{R}^n\),
\begin{equation}\label{eq:matrix-exp-bounds}
e^{m(A)t}\|z\|
\le
\|e^{At}z\|
\le
e^{M(A)t}\|z\|.
\end{equation}
\end{lemma}

\begin{proof}
Let
\begin{equation}\label{eq:app-matrix-exp-trajectory}
x(t):=\mathrm{e}^{At}z,
\qquad t\ge 0.
\end{equation}
Then \(x'(t)=Ax(t)\). Writing \(S:=\sym A\), we compute
\begin{align}
\frac12\frac{d}{dt}\|x(t)\|^2
&=
x(t)^\top A x(t) \nonumber\\
&=
x(t)^\top S x(t). \label{eq:app-energy-identity}
\end{align}
By the spectral bounds for the symmetric matrix \(S\),
\begin{equation}\label{eq:app-energy-sandwich}
m(A)\|x(t)\|^2
\le
x(t)^\top S x(t)
\le
M(A)\|x(t)\|^2.
\end{equation}
Hence
\begin{equation}\label{eq:app-y-ineq}
2m(A)\|x(t)\|^2
\le
\frac{d}{dt}\|x(t)\|^2
\le
2M(A)\|x(t)\|^2.
\end{equation}
Define
\begin{equation}\label{eq:app-y-def}
y(t):=\|x(t)\|^2.
\end{equation}
Then
\begin{equation}\label{eq:app-y-diff}
2m(A)y(t)\le y'(t)\le 2M(A)y(t).
\end{equation}
If \(y(0)=0\), then \(z=0\), hence \(x(t)\equiv 0\), and the claim is trivial. Assume now that \(y(0)>0\). Since \(y\) is continuous and satisfies \eqref{eq:app-y-diff}, it remains nonnegative for all \(t\ge 0\), and on every interval where \(y(t)>0\) we may divide by \(y(t)\) and integrate from \(0\) to \(t\) to obtain
\begin{equation}\label{eq:app-log-y}
2m(A)t
\le
\log y(t)-\log y(0)
\le
2M(A)t.
\end{equation}
Exponentiating gives
\begin{equation}\label{eq:app-y-exp}
\mathrm{e}^{2m(A)t}y(0)
\le
y(t)
\le
\mathrm{e}^{2M(A)t}y(0).
\end{equation}
Since \(y(0)=\|z\|^2\), taking square roots proves \eqref{eq:matrix-exp-bounds}.
\end{proof}

\begin{theorem}[LTI specialization of the nonlinear certificates]
\label{thm:lti-specialization}
Assume that \eqref{eq:ode} is given by \eqref{eq:matrix}, and let \(\hat x\) be a sufficiently smooth PINN approximation. Suppose that
\begin{equation}\label{eq:lti-residual-bound}
\|\mathcal{R}_{\hat\varphi}(t)\|
\le
\delta(t),
\quad t\in\mathbb{T},
\end{equation}
for some continuous function \(\delta:\mathbb{T}\to\mathbb{R}_+\). Then, for every \(t\in\mathbb{T}\),
\begin{align}
\|e(t)\|
&\ge
e^{m(A)t}\|e(0)\|
-
\int_0^t e^{m(A)(t-s)}\,\delta(s)\,ds,
\label{eq:lti-lower}
\\[0.5ex]
\|e(t)\|
&\le
e^{M(A)t}\|e(0)\|
+
\int_0^t e^{M(A)(t-s)}\,\delta(s)\,ds.
\label{eq:lti-upper}
\end{align}
If, in addition, \(\delta(t)\equiv\bar\delta\) is constant, then
\begin{align}
\|e(t)\|
&\ge
e^{m(A)t}\|e(0)\|
-
\Psi_{m(A)}(t)\,\bar\delta,
\label{eq:lti-lower-constant-delta}
\\[0.5ex]
\|e(t)\|
&\le
e^{M(A)t}\|e(0)\|
+
\Psi_{M(A)}(t)\,\bar\delta,
\label{eq:lti-upper-constant-delta}
\end{align}
where
\begin{equation}\label{eq:PsiM}
\Psi_\alpha(t):=
\begin{cases}
\dfrac{e^{\alpha t}-1}{\alpha}, & \alpha\neq 0,\\[1.2ex]
t, & \alpha=0.
\end{cases}
\end{equation}
\end{theorem}

\begin{proof}
For the linear time-invariant system \(f(t,x)=Ax\), the error
\begin{equation}\label{eq:app-lti-error-def}
e(t):=x(t)-\hat x(t)
\end{equation}
satisfies
\begin{equation}\label{eq:app-lti-error}
\dot e(t)
=
Ae(t)-\mathcal R_{\hat\varphi}(t).
\end{equation}
Because \(e\in C^1(\mathbb{T};\mathbb{R}^n)\), the map \(t\mapsto \|e(t)\|\) is absolutely continuous. Hence, for almost every \(t\in\mathbb{T}\) such that \(e(t)\neq 0\),
\begin{equation}\label{eq:app-lti-norm-derivative}
\frac{d}{dt}\|e(t)\|
=
\frac{e(t)^\top A e(t)}{\|e(t)\|}
-
\frac{e(t)\cdot \mathcal R_{\hat\varphi}(t)}{\|e(t)\|}.
\end{equation}
Since
\begin{equation}\label{eq:app-lti-sym-part}
e(t)^\top A e(t)=e(t)^\top (\sym A)e(t),
\end{equation}
\Cref{lem:linear-one-sided} yields
\begin{equation}\label{eq:app-lti-sym-bounds}
m(A)\|e(t)\|^2
\le
e(t)^\top A e(t)
\le
M(A)\|e(t)\|^2.
\end{equation}
Moreover, by Cauchy--Schwarz and \eqref{eq:lti-residual-bound},
\begin{equation}\label{eq:app-lti-residual-bound}
-\|e(t)\|\,\delta(t)
\le
-e(t)\cdot \mathcal R_{\hat\varphi}(t)
\le
\|e(t)\|\,\delta(t).
\end{equation}
Combining \eqref{eq:app-lti-norm-derivative}, \eqref{eq:app-lti-sym-bounds}, and \eqref{eq:app-lti-residual-bound}, we obtain
\begin{equation}\label{eq:app-lti-lower-diff}
\frac{d}{dt}\|e(t)\|
\ge
m(A)\|e(t)\|-\delta(t)
\qquad \text{for a.e. } t\in\mathbb{T},
\end{equation}
and
\begin{equation}\label{eq:app-lti-upper-diff}
\frac{d}{dt}\|e(t)\|
\le
M(A)\|e(t)\|+\delta(t)
\qquad \text{for a.e. } t\in\mathbb{T}.
\end{equation}
Multiplying \eqref{eq:app-lti-lower-diff} by \(\mathrm e^{-m(A)t}\) and integrating from \(0\) to \(t\) yields \eqref{eq:lti-lower}. Similarly, multiplying \eqref{eq:app-lti-upper-diff} by \(\mathrm e^{-M(A)t}\) and integrating from \(0\) to \(t\) yields \eqref{eq:lti-upper}. If, in addition, \(\delta(t)\equiv \bar\delta\) is constant, then
\begin{equation}\label{eq:app-lti-psi-general}
\int_0^t \mathrm e^{\alpha(t-s)}\,ds
=
\Psi_\alpha(t)
\end{equation}
for every \(\alpha\in\mathbb{R}\), where \(\Psi_\alpha\) is defined in \eqref{eq:PsiM}. Applying \eqref{eq:app-lti-psi-general} with \(\alpha=m(A)\) in \eqref{eq:lti-lower} and with \(\alpha=M(A)\) in \eqref{eq:lti-upper} proves \eqref{eq:lti-lower-constant-delta} and \eqref{eq:lti-upper-constant-delta}.
\end{proof}

\begin{corollary}[LTV specialization]
\label{cor:ltv-specialization}
Consider the linear time-varying system
\begin{equation}\label{eq:ltv-system}
\dot x(t)=A(t)x(t)+b(t),
\end{equation}
where \(A:[0,T]\to\mathbb{R}^{n\times n}\) is continuous. Define
\begin{equation}\label{eq:ltv-mM}
m(t):=\lambda_{\min}\!\bigl(\sym A(t)\bigr),
\quad
M(t):=\lambda_{\max}\!\bigl(\sym A(t)\bigr),
\end{equation}
with
\begin{equation}\label{eq:symAt}
\sym A(t):=\frac{A(t)+A(t)^\top}{2}.
\end{equation}
If
\begin{equation}\label{eq:ltv-residual-bound}
\|\mathcal{R}_{\hat\varphi}(t)\|
\le
\delta(t),
\quad t\in\mathbb{T},
\end{equation}
then the localized nonlinear bounds from \Cref{thm:localized-lower,thm:localized-upper} specialize to
\begin{align}
\|e(t)\|
&\ge
\exp\!\Bigl(\int_0^t m(\tau)\,d\tau\Bigr)\|e(0)\|\nonumber\\
&\quad-
\int_0^t
\exp\!\Bigl(\int_s^t m(\tau)\,d\tau\Bigr)\delta(s)\,ds,
\label{eq:ltv-lower}
\\[0.5ex]
\|e(t)\|
&\le
\exp\!\Bigl(\int_0^t M(\tau)\,d\tau\Bigr)\|e(0)\|\nonumber\\
&\quad+
\int_0^t
\exp\!\Bigl(\int_s^t M(\tau)\,d\tau\Bigr)\delta(s)\,ds.
\label{eq:ltv-upper}
\end{align}
\end{corollary}

\begin{proof}
For the linear time-varying system \eqref{eq:ltv-system}, define
\begin{equation}\label{eq:app-ltv-error}
e(t):=x(t)-\hat x(t).
\end{equation}
Since the residual is
\begin{equation}\label{eq:app-ltv-residual}
\mathcal R_{\hat\varphi}(t)
=
\dot{\hat x}(t)-A(t)\hat x(t)-b(t),
\end{equation}
the error satisfies
\begin{equation}\label{eq:app-ltv-error-dynamics}
\dot e(t)
=
A(t)e(t)-\mathcal R_{\hat\varphi}(t).
\end{equation}
For almost every \(t\in\mathbb{T}\) such that \(e(t)\neq 0\),
\begin{equation}\label{eq:app-ltv-norm-derivative}
\frac{d}{dt}\|e(t)\|
=
\frac{e(t)^\top A(t)e(t)}{\|e(t)\|}
-
\frac{e(t)\cdot \mathcal R_{\hat\varphi}(t)}{\|e(t)\|}.
\end{equation}
Since
\begin{equation}\label{eq:app-ltv-sym-part}
e(t)^\top A(t)e(t)
=
e(t)^\top (\sym A(t))e(t),
\end{equation}
the Rayleigh--Ritz characterization of the extremal eigenvalues of \(\sym A(t)\) gives
\begin{equation}\label{eq:app-ltv-sym-bounds}
m(t)\|e(t)\|^2
\le
e(t)^\top A(t)e(t)
\le
M(t)\|e(t)\|^2.
\end{equation}
Also, by Cauchy--Schwarz and \eqref{eq:ltv-residual-bound},
\begin{equation}\label{eq:app-ltv-residual-estimate}
-\|e(t)\|\,\delta(t)
\le
-e(t)\cdot \mathcal R_{\hat\varphi}(t)
\le
\|e(t)\|\,\delta(t).
\end{equation}
Therefore,
\begin{equation}\label{eq:app-ltv-lower-diff}
\frac{d}{dt}\|e(t)\|
\ge
m(t)\|e(t)\|-\delta(t)
\qquad \text{for a.e. } t\in\mathbb{T},
\end{equation}
and
\begin{equation}\label{eq:app-ltv-upper-diff}
\frac{d}{dt}\|e(t)\|
\le
M(t)\|e(t)\|+\delta(t)
\qquad \text{for a.e. } t\in\mathbb{T}.
\end{equation}
Define
\begin{equation}\label{eq:app-ltv-lambda}
\Lambda_m(t):=\int_0^t m(\tau)\,d\tau,
\qquad
\Lambda_M(t):=\int_0^t M(\tau)\,d\tau.
\end{equation}
Multiplying \eqref{eq:app-ltv-lower-diff} by \(\mathrm e^{-\Lambda_m(t)}\) and integrating from \(0\) to \(t\) yields \eqref{eq:ltv-lower}. Similarly, multiplying \eqref{eq:app-ltv-upper-diff} by \(\mathrm e^{-\Lambda_M(t)}\) and integrating gives \eqref{eq:ltv-upper}.
\end{proof}

\begin{remark}[Linear analogue of the nonlinear constants]
\label{rem:linear-three-constants}
The linear case provides the exact analogue of the nonlinear constants \(\ell_D\), \(\mu_D\), and \(L_D\). More precisely,
\[
\ell_D \longleftrightarrow m(A),
\qquad
\mu_D \longleftrightarrow M(A),
\qquad
L_D \longleftrightarrow \|A\|_2.
\]
Thus, in the linear setting, the sharp lower one-sided Lipschitz constant is \(m(A)\), the sharp upper one-sided Lipschitz constant is \(M(A)\), and the sharp Euclidean Lipschitz constant is \(\|A\|_2\). 
\end{remark}

\begin{remark}[Relation to Lipschitz-based upper bounds]
\label{rem:comparison-with-upper-bounds}
In the nonlinear setting, the upper certificate is governed by an upper one-sided Lipschitz constant \(\mu_D\), whereas earlier upper-bound results are formulated in terms of a Lipschitz constant \(L_D\). Since every Lipschitz constant is also an admissible upper one-sided Lipschitz constant, the one-sided framework is weaker as an assumption and can yield sharper upper certificates whenever \(\mu_D<L_D\).

The linear case is the exact analogue of this comparison. For \(f(x)=Ax\), the sharp upper one-sided Lipschitz constant is
\begin{equation}
M(A)=\lambda_{\max}\!\Bigl(\frac{A+A^\top}{2}\Bigr),
\end{equation}
whereas the sharp Euclidean Lipschitz constant is \(\|A\|_2\). Indeed, for every \(z\in\mathbb{R}^n\),
\begin{equation}
z^\top (\sym A)z
=
(Az)\cdot z
\le
\|Az\|\,\|z\|
\le
\|A\|_2\,\|z\|^2,
\end{equation}
so that
\begin{equation}
M(A)\le \|A\|_2.
\end{equation}
In general, the inequality is strict for non-normal matrices.

It is important to distinguish between the sharp constants and merely admissible ones. In the upper one-sided framework, any constant \(\bar\lambda\ge M(A)\) is admissible in \eqref{eq:lti-upper}. By contrast, in the Euclidean Lipschitz-based framework, admissibility requires \(L\ge \|A\|_2\). Therefore, whenever \(M(A)<\|A\|_2\), there exist intermediate values
\begin{equation}
\bar\lambda\in [M(A),\|A\|_2)
\end{equation}
which are valid upper one-sided constants but are \emph{not} valid Lipschitz constants. Thus, the upper one-sided formulation is never worse than the Lipschitz-based one when both are compared through their sharp Euclidean constants, and it can be strictly sharper.

By contrast, the lower certificate in the nonlinear case depends on any admissible lower one-sided Lipschitz constant \(\ell_D\). When \(\ell_D>0\), this is a strong monotonicity constant. In the linear case, the sharp largest admissible lower one-sided Lipschitz constant is
\begin{equation}
m(A)=\lambda_{\min}\!\Bigl(\frac{A+A^\top}{2}\Bigr).
\end{equation}
Equivalently, every constant \(\ell\le m(A)\) is admissible in the linear lower estimate, while \(m(A)\) is the optimal choice.
\end{remark}

\begin{remark}[Geometric enclosure of the exact solution]
\label{rem:set-valued-enclosure}
Whenever computable functions \(\underline E(t)\) and \(\overline E(t)\) satisfy
\begin{equation}\label{eq:two-sided-error-band}
\underline E(t)\le \|e(t)\|\le \overline E(t),
\end{equation}
the exact solution is localized to the set
\begin{equation}\label{eq:Xcert}
\mathcal{X}_{\mathrm{cert}}(t)
:=
\bigl\{
z\in\mathbb{R}^n:
\underline E(t)\le \|z-\hat x(t)\|\le \overline E(t)
\bigr\}.
\end{equation}
Equivalently,
\begin{equation}\label{eq:Xcert-ball-form}
\mathcal{X}_{\mathrm{cert}}(t)
=
\overline B_2\bigl(\hat x(t),\overline E(t)\bigr)
\setminus
B_2\bigl(\hat x(t),\underline E(t)\bigr).
\end{equation}
Thus, the lower and upper certificates do not determine the exact solution uniquely, but they localize it to a computable Euclidean enclosure centered at the PINN approximation.

Since the linear results above are global on \(\mathbb{R}^n\), no additional state-space restriction is required. If independent prior information yields a relevant admissible domain \(D\subset\mathbb{R}^n\), one may further refine the location set to \(\mathcal{X}_{\mathrm{cert}}(t)\cap D\).

For \(n=1\), the enclosure is the union of two closed intervals, or a single interval if \(\underline E(t)=0\). For \(n=2\), it is a closed annulus; for \(n=3\), a closed spherical shell; and for general \(n\ge 2\), a closed annular region in \(\mathbb{R}^n\).
\end{remark}

\subsection{Auxiliary numerical certification by the classical fourth-order Runge--Kutta method}
\label{sec:rk4-certification}

To evaluate the convolution terms appearing in the constant-coefficient lower and upper certificates, we use a single numerical integration device throughout the paper: the classical fourth-order Runge--Kutta method (RK4). This yields a simpler and more uniform certification pipeline.

We consider a constant certification coefficient \(a\in\mathbb{R}\) on a fixed certified time window \([0,T_{\mathrm{cert}}]\). This covers directly the nonlinear constant-coefficient certificates with \(a=\ell_D\) or \(a=\mu_D\), and the linear time-invariant certificates with \(a=m(A)\) or \(a=M(A)\). For such a constant \(a\), define
\begin{equation}
I_a(t,\delta):=\int_0^t e^{a(t-s)}\delta(s)\,ds,
\qquad t\in[0,T_{\mathrm{cert}}].
\end{equation}
This quantity appears in the lower and upper certificates depending on whether one uses a lower coefficient \(a=\ell_D\) or \(a=m(A)\), or an upper coefficient \(a=\mu_D\) or \(a=M(A)\).

The key observation is that \(I_a(\cdot,\delta)\) is the unique solution of the scalar auxiliary initial value problem
\begin{equation}\label{eq:aux-rk4}
\begin{cases}
\dot z(t)=a z(t)+\delta(t), \qquad t\in[0,T_{\mathrm{cert}}],\\
z(0)=0.
\end{cases}
\end{equation}
Hence the numerical evaluation of the constant-coefficient certification integrals reduces to the numerical solution of \eqref{eq:aux-rk4}.

Let \(0=t_0<t_1<\cdots<t_N=T_{\mathrm{cert}}\) be a uniform mesh with step size
\begin{equation}
h=\frac{T_{\mathrm{cert}}}{N}.
\end{equation}
Starting from \(z_0=0\), the classical RK4 scheme is
\begin{align*}
k_1^{(n)} &= a z_n + \delta(t_n),\\
k_2^{(n)} &= a\Bigl(z_n+\frac{h}{2}k_1^{(n)}\Bigr)
            +\delta\Bigl(t_n+\frac{h}{2}\Bigr),\\
k_3^{(n)} &= a\Bigl(z_n+\frac{h}{2}k_2^{(n)}\Bigr)
            +\delta\Bigl(t_n+\frac{h}{2}\Bigr),\\
k_4^{(n)} &= a\bigl(z_n+h\,k_3^{(n)}\bigr)
            +\delta(t_n+h),
\end{align*}
and
\begin{equation}
z_{n+1}
=
z_n+\frac{h}{6}\Bigl(k_1^{(n)}+2k_2^{(n)}+2k_3^{(n)}+k_4^{(n)}\Bigr).
\end{equation}
We denote the corresponding RK4 approximation at the mesh points by
\begin{equation}
\widehat I_{a,N}(t_n,\delta):=z_n.
\end{equation}

\begin{theorem}[Computable RK4 remainder for the certification IVP]
\label{thm:rk4-computable}
Let \(z\) solve \eqref{eq:aux-rk4}, and assume that \(\delta\in C^4([0,T_{\mathrm{cert}}])\). For \(q=0,1,2,3,4\), define
\begin{equation}
D_q(T_{\mathrm{cert}})
:=
\max_{t\in[0,T_{\mathrm{cert}}]}
|\delta^{(q)}(t)|.
\end{equation}
Set
\begin{equation}
\Xi_a(T_{\mathrm{cert}})
:=
\begin{cases}
\dfrac{e^{|a|T_{\mathrm{cert}}}-1}{|a|}, & a\neq 0,\\[0.8em]
T_{\mathrm{cert}}, & a=0,
\end{cases}
\end{equation}
\begin{equation}
\overline z(T_{\mathrm{cert}})
:=
\Xi_a(T_{\mathrm{cert}})\,D_0(T_{\mathrm{cert}}),
\end{equation}
and
\begin{equation}
K_5(T_{\mathrm{cert}})
:=
|a|^5\overline z(T_{\mathrm{cert}})
+
\sum_{q=0}^{4}|a|^{4-q}D_q(T_{\mathrm{cert}}).
\end{equation}
Then the RK4 approximation satisfies the explicit global bound
\begin{equation}
|I_a(T_{\mathrm{cert}},\delta)-\widehat I_{a,N}(T_{\mathrm{cert}},\delta)|
\le
E_{\mathrm{RK4}}^{a}(T_{\mathrm{cert}},N),
\end{equation}
where
\begin{equation}
E_{\mathrm{RK4}}^{a}(T_{\mathrm{cert}},N)
:=
\frac{\Xi_a(T_{\mathrm{cert}})}{90}\,
K_5(T_{\mathrm{cert}})\,h^4.
\end{equation}
Equivalently,
\begin{equation}
|I_a(T_{\mathrm{cert}},\delta)-\widehat I_{a,N}(T_{\mathrm{cert}},\delta)|
\le
C_{\mathrm{RK4}}^{a}(T_{\mathrm{cert}})\,h^4,
\end{equation}
with
\begin{equation}
C_{\mathrm{RK4}}^{a}(T_{\mathrm{cert}})
:=
\frac{\Xi_a(T_{\mathrm{cert}})}{90}\,
K_5(T_{\mathrm{cert}}).
\end{equation}
\end{theorem}

\begin{proof}
Let \(z\) solve
\begin{equation}
\dot z(t)=a z(t)+\delta(t),
\qquad
z(0)=0.
\end{equation}
Since \(\delta\in C^4([0,T_{\mathrm{cert}}])\), it follows from the differential equation that \(z\in C^5([0,T_{\mathrm{cert}}])\). By variation of constants,
\begin{equation}
z(t)=\int_0^t e^{a(t-s)}\delta(s)\,ds.
\end{equation}
Hence for all $t\in[0,T_{\mathrm{cert}}]$,
\begin{align}
|z(t)|
&\le
\int_0^t e^{|a|(t-s)}|\delta(s)|\,ds\nonumber\\
&\le
\Xi_a(T_{\mathrm{cert}})\,D_0(T_{\mathrm{cert}})
=
\overline z(T_{\mathrm{cert}}).
\end{align}
Repeated differentiation of the auxiliary equation gives
\begin{align}
z^{(5)}(t)
&=
a^5 z(t)
+a^4\delta(t)
+a^3\delta'(t)\nonumber\\
&\quad+a^2\delta''(t)
+a\delta'''(t)
+\delta^{(4)}(t).
\end{align}
Therefore,
\begin{equation}
|z^{(5)}(t)|
\le
|a|^5\overline z(T_{\mathrm{cert}})
+
\sum_{q=0}^4 |a|^{4-q}D_q(T_{\mathrm{cert}})
=
K_5(T_{\mathrm{cert}}).
\end{equation}

Let \(z_n\) be the RK4 approximation on the uniform mesh \(t_n=nh\), where \(h=T_{\mathrm{cert}}/N\). For one RK4 step, the local truncation error satisfies
\begin{align*}
|z(t_{n+1})-\Phi_h(t_n,z(t_n))|
&\le
\frac{h^5}{90}
\max_{s\in[t_n,t_{n+1}]} |z^{(5)}(s)|\\
&\le
\frac{h^5}{90}K_5(T_{\mathrm{cert}}).
\end{align*}
Set
\begin{equation}
e_n:=z(t_n)-z_n.
\end{equation}
Since \(z_0=z(0)=0\), we have \(e_0=0\). For the scalar linear equation, the RK4 propagation factor is
\begin{equation}
R(\zeta)
=
1+\zeta+\frac{\zeta^2}{2}
+\frac{\zeta^3}{6}
+\frac{\zeta^4}{24},
\end{equation}
so that
\begin{equation}
e_{n+1}=R(a h)e_n+\tau_n,
\qquad
|\tau_n|\le \frac{h^5}{90}K_5(T_{\mathrm{cert}}).
\end{equation}
Since \(|R(a h)|\le e^{|a|h}\), iteration yields
\begin{equation}
|e_N|
\le
\frac{h^5}{90}K_5(T_{\mathrm{cert}})
\sum_{j=0}^{N-1}e^{|a|jh}.
\end{equation}
Using
\begin{equation}
h\sum_{j=0}^{N-1}e^{|a|jh}
\le
\int_0^{T_{\mathrm{cert}}} e^{|a|s}\,ds
=
\Xi_a(T_{\mathrm{cert}}),
\end{equation}
we obtain
\begin{equation}
|e_N|
\le
\frac{\Xi_a(T_{\mathrm{cert}})}{90}
K_5(T_{\mathrm{cert}})\,h^4.
\end{equation}
Since \(z(T_{\mathrm{cert}})=I_a(T_{\mathrm{cert}},\delta)\) and \(z_N=\widehat I_{a,N}(T_{\mathrm{cert}},\delta)\), this proves the claim.
\end{proof}

\begin{remark}[Use in lower and upper certification]
\Cref{thm:rk4-computable} applies to every auxiliary equation of the form \eqref{eq:aux-rk4}. Thus the same RK4 remainder formula can be used for both lower and upper constant-coefficient certificates. Only the coefficient \(a\) changes:
\begin{equation}
a=
\begin{cases}
\ell_D \text{ or } m(A), & \text{for lower certificates},\\
\mu_D \text{ or } M(A), & \text{for upper certificates}.
\end{cases}
\end{equation}
\end{remark}

\begin{corollary}[Certificates with explicit RK4 remainder]
\label{cor:rk4-certified}
Let \(t\in(0,T_{\mathrm{cert}}]\), and let \(\widehat I_{a,N}(t,\delta)\) denote the RK4 approximation of \(I_a(t,\delta)\) computed on a uniform mesh of \([0,t]\). Let \(E_{\mathrm{RK4}}^{a}(t,N)\) be the corresponding remainder from \Cref{thm:rk4-computable}, with \(T_{\mathrm{cert}}\) replaced by \(t\). Then
\begin{equation}
I_a(t,\delta)
\le
\widehat I_{a,N}(t,\delta)+E_{\mathrm{RK4}}^{a}(t,N),
\end{equation}
and
\begin{equation}
I_a(t,\delta)
\ge
\widehat I_{a,N}(t,\delta)-E_{\mathrm{RK4}}^{a}(t,N).
\end{equation}
Consequently, for constant lower and upper coefficients \(\ell_D\) and \(\mu_D\), the certified bounds become
\begin{equation}
\|e(t)\|
\ge
e^{\ell_D t}\|e(0)\|
-
\widehat I_{\ell_D,N}(t,\delta)
-
E_{\mathrm{RK4}}^{\ell_D}(t,N),
\end{equation}
and
\begin{equation}
\|e(t)\|
\le
e^{\mu_D t}\|e(0)\|
+
\widehat I_{\mu_D,N}(t,\delta)
+
E_{\mathrm{RK4}}^{\mu_D}(t,N).
\end{equation}
Analogous formulas hold in the linear time-invariant case with \(\ell_D=m(A)\) and \(\mu_D=M(A)\).
\end{corollary}

\begin{proof}
The absolute RK4 error bound implies
\begin{equation}
I_a(t,\delta)
\le
\widehat I_{a,N}(t,\delta)+E_{\mathrm{RK4}}^{a}(t,N).
\end{equation}
For the lower certificate, the convolution term enters with a minus sign. Therefore,
\begin{equation}
-I_a(t,\delta)
\ge
-\widehat I_{a,N}(t,\delta)-E_{\mathrm{RK4}}^{a}(t,N).
\end{equation}
Substituting this into the analytical lower certificate gives the stated lower bound.

For the upper certificate, the convolution term enters with a plus sign, so the same upper estimate gives directly
\begin{equation}
I_a(t,\delta)
\le
\widehat I_{a,N}(t,\delta)+E_{\mathrm{RK4}}^{a}(t,N).
\end{equation}
Substituting this into the analytical upper certificate gives the stated upper bound.
\end{proof}

\begin{remark}[Fully rigorous versus conditionally rigorous certification]
\Cref{thm:rk4-computable} is fully rigorous if the derivative suprema \(D_q(T_{\mathrm{cert}})\) are themselves certified analytically or by verified differentiation. If the quantities \(D_q(T_{\mathrm{cert}})\) are only estimated numerically from sampled data or from an unconstrained smooth fit, then the resulting RK4 remainder should be described as conditionally rigorous.
\end{remark}

\paragraph{Choosing the RK4 mesh width.}
Fix a certified time window \([0,T_{\mathrm{cert}}]\), a tolerance parameter \(\eta\in(0,1)\), and a positive reference scale \(S_*>0\). To guarantee that the numerical post-processing error remains below the fraction \(\eta\) of that scale, it is sufficient to impose
\begin{equation}
E_{\mathrm{RK4}}^{a}(T_{\mathrm{cert}},N)\le \eta S_*.
\end{equation}
Since
\begin{equation}
E_{\mathrm{RK4}}^{a}(T_{\mathrm{cert}},N)
=
C_{\mathrm{RK4}}^{a}(T_{\mathrm{cert}})
\left(\frac{T_{\mathrm{cert}}}{N}\right)^4,
\end{equation}
it suffices to choose
\begin{equation}
N
\ge
\max\left\{
1,\,
\left\lceil
T_{\mathrm{cert}}
\left(
\frac{C_{\mathrm{RK4}}^{a}(T_{\mathrm{cert}})}
{\eta S_*}
\right)^{1/4}
\right\rceil
\right\}.
\end{equation}
If both lower and upper certificates are computed on the same time window, one may use
\begin{equation}
N=\max\{N_{\ell_D},N_{\mu_D}\}
\end{equation}
in the nonlinear constant-coefficient case, or
\begin{equation}
N=\max\{N_{m(A)},N_{M(A)}\}
\end{equation}
in the linear time-invariant case.

\begin{remark}[Practical adaptive choice of the reference scale]
For implementation, one may choose the reference scale from a pilot certificate, for example
\begin{equation}
S_*:=\max\left\{
\inf_{t\in[0,T_{\mathrm{cert}}]}
\bigl(\overline E_{\mathrm{pilot}}(t)-\underline E_{\mathrm{raw,pilot}}(t)\bigr),
\,10^{-12}
\right\}.
\end{equation}
This makes the RK4 post-processing error negligible relative to the final certified band width, while keeping the a priori mesh rule explicit.
\end{remark}

\section{Application to Physics-Informed Neural Networks}
\label{Sec 4}

PINNs incorporate the governing differential equation into the training objective through residual minimization \cite{raissi2019physics,karniadakis2021physicsinformed}. In the present setting, the aim is to approximate the local flow map
\begin{equation}
\varphi:\mathbb{T}\times D_{\mathrm{adm}}\to D
\end{equation}
introduced in \Cref{Sec 2} by a neural network
\begin{equation}
\hat{\varphi}_\theta:\mathbb{T}\times D_{\mathrm{adm}}\to\mathbb{R}^n,
\end{equation}
where \(\theta\in\mathbb{R}^k\) denotes the trainable parameter vector.

Since exact solution data are typically unavailable, training is based on collocation in time and initial-value space. Let
\begin{equation}
Y_{\mathrm{coll}}\subset \mathbb{T}\times D_{\mathrm{adm}}
\end{equation}
be a finite set of collocation points \(y=(t,\xi)\). The physics loss is then defined by
\begin{equation}
\mathcal L_{\mathrm{physics}}(\theta)
:=
\frac{1}{|Y_{\mathrm{coll}}|}
\sum_{(t,\xi)\in Y_{\mathrm{coll}}}
\left\|
\partial_t \hat{\varphi}_\theta(t,\xi)
-
f\bigl(t,\hat{\varphi}_\theta(t,\xi)\bigr)
\right\|^2.
\end{equation}
This loss penalizes violations of the ODE and is the training analogue of the residual quantity appearing in the \emph{a posteriori} certification results.

For initial-value problems, two enforcement strategies are common.

\medskip
\noindent\textbf{Soft-constrained PINNs.}
In the soft-constrained setting, the initial condition is imposed through an additional penalty term. Given training initial values
\begin{equation}
Y_{\mathrm{ic}}=\{\xi^{(1)},\dots,\xi^{(M)}\}\subset D_{\mathrm{adm}},
\end{equation}
one defines
\begin{equation}
\mathcal L_{\mathrm{initial}}(\theta)
:=
\frac{1}{|Y_{\mathrm{ic}}|}
\sum_{\xi\in Y_{\mathrm{ic}}}
\left\|
\hat{\varphi}_\theta(0,\xi)-\xi
\right\|^2.
\end{equation}
The total loss takes the form
\begin{equation}\label{eq:pinn-soft-loss}
\mathcal L(\theta)
=
\gamma_{\mathrm{physics}}\mathcal L_{\mathrm{physics}}(\theta)
+
\gamma_{\mathrm{initial}}\mathcal L_{\mathrm{initial}}(\theta),
\end{equation}
where \(\gamma_{\mathrm{physics}},\gamma_{\mathrm{initial}}>0\) are weighting parameters.

\medskip
\noindent\textbf{Hard-constrained PINNs.}
In the hard-constrained setting, the network ansatz is chosen so that the initial condition is satisfied exactly by construction \cite{lu2021hardconstraints}. In that case,
\begin{equation}
\hat{\varphi}_\theta(0,\xi)=\xi
\qquad \text{for all } \xi\in D_{\mathrm{adm}},
\end{equation}
and no separate initial-condition penalty is required. The training objective then reduces to the physics loss, possibly combined with additional regularization terms.

\medskip
\noindent\textbf{Connection with the certified bounds.}
After training, let \(\theta^*\) denote the learned parameter vector, and fix an initial value \(x_0\in D_{\mathrm{adm}}\). The corresponding PINN trajectory is
\begin{equation}
\hat{x}(t):=\hat{\varphi}_{\theta^*}(t,x_0),
\end{equation}
with residual
\begin{equation}
\mathcal R_{\hat\varphi}(t)
=
\dot{\hat{x}}(t)-f\bigl(t,\hat{x}(t)\bigr).
\end{equation}
The certification results of \Cref{Sec 3} apply to this trajectory once a computable residual majorant \(\delta\) is available such that
\begin{equation}
\|\mathcal R_{\hat\varphi}(t)\|\le \delta(t),
\qquad t\in\mathbb{T}.
\end{equation}

In the soft-constrained case, the anchoring error
\begin{equation}
\|e(0)\|=\|x_0-\hat{x}(0)\|
\end{equation}
is directly computable from the trained network and enters both the lower and upper certificates. In the hard-constrained case, one has
\begin{equation}
e(0)=0.
\end{equation}
As a consequence, the scalar lower certificate from \Cref{Sec 3} may become trivial after nonnegativity truncation. In the linear setting, this loss of lower information can be repaired by the signed-residual finite-probe certificate in \Cref{thm:signed-residual-probe}. 

Thus, the proposed certification framework is compatible with both common PINN training paradigms. Soft-constrained PINNs naturally provide a computable initial mismatch for the scalar two-sided formulas. Hard-constrained PINNs may require directional residual information, such as the finite-probe certificate in the linear case, to obtain nontrivial lower diagnostics.

\section{Certificate-Informed Training via Upper Error Bounds}
\label{Sec 5}

The certificates above are primarily \emph{a posteriori} evaluation tools. Nevertheless, the upper certificate can also be used constructively during training. This use must be separated from lower-bound evaluation. The upper estimator is monotone in both the initial mismatch and the residual majorant, whereas the scalar lower estimator can be driven to the trivial value zero or can reward a larger initial mismatch if it is maximized. Thus the lower certificate is kept as a diagnostic, while the upper certificate can be used as an auxiliary regularizer.

Let \(\hat x_\theta\) be a PINN approximation with residual
\begin{equation}
 R_\theta(t)=\dot{\hat x}_\theta(t)-f(t,\hat x_\theta(t)),
\end{equation}
and let \(\delta_\theta(t)\ge \|R_\theta(t)\|\) be a smooth computable residual majorant. For a constant \(c\) equal either to an upper one-sided Lipschitz constant \(\mu_D\) or to a Lipschitz constant \(L_D\), define
\begin{equation}
 U_c^\theta(t)=e^{ct}\|x_0-\hat x_\theta(0)\|+
 \int_0^t e^{c(t-s)}\delta_\theta(s)\,ds.
\label{eq:training-upper-certificate}
\end{equation}

\begin{theorem}[Auxiliary ODE representation of the upper certificate]
\label{thm:auxiliary-training-ode}

For fixed \(\theta\) and fixed \(c\), the function \(U_c^\theta\) in \eqref{eq:training-upper-certificate} is the unique solution of
\begin{equation}
 (z_c^\theta)'(t)=c z_c^\theta(t)+\delta_\theta(t),
 \qquad
 z_c^\theta(0)=\|x_0-\hat x_\theta(0)\|.
\label{eq:training-aux-ode}
\end{equation}
\end{theorem}

\begin{proof}
Differentiating \eqref{eq:training-upper-certificate} gives
\begin{align}
\frac{d}{dt}U_c^\theta(t)
&=c e^{ct}\|x_0-\hat x_\theta(0)\|+
\delta_\theta(t)+c\int_0^t e^{c(t-s)}\delta_\theta(s)\,ds \\
&=c U_c^\theta(t)+\delta_\theta(t).
\end{align}
The initial condition follows by setting \(t=0\). Uniqueness follows from the standard uniqueness theorem for scalar linear ODEs.
\end{proof}

The auxiliary ODE formulation gives the training loss
\begin{equation}
 \mathcal L_{\mathrm{stage\,2}}(\theta)
 =\mathcal L_{\mathrm{PINN}}(\theta)
 +\eta\,\mathcal L_{\mathrm{cert}}^{\mathrm{up}}(\theta),
 \qquad \eta>0,
\label{eq:stage2-loss}
\end{equation}
where, on a sorted auxiliary grid \(Y=\{t_i\}_{i=0}^N\), one may choose
\begin{equation}
 \mathcal L_{\mathrm{cert}}^{\mathrm{up}}(\theta)
 =\frac{1}{N+1}\sum_{i=0}^N z_c^\theta(t_i)
 \quad\text{or}\quad
 \mathcal L_{\mathrm{cert}}^{\mathrm{up}}(\theta)=z_c^\theta(T).
\end{equation}
The auxiliary state is recomputed from the current network at each training step. It is not an additional neural-network output.

\begin{theorem}[One-sided upper certificates are no larger than Lipschitz certificates]
\label{thm:one-sided-training-comparison}
Assume \(\mu_D\le L_D\). For the same residual majorant \(\delta_\theta\) and the same initial mismatch, the propagated upper certificates satisfy
\begin{equation}
 U_{\mu_D}^\theta(t)\le U_{L_D}^\theta(t),
 \qquad t\in[0,T].
\end{equation}
In particular, the upper one-sided training regularizer is never more conservative than the Lipschitz-based regularizer when both are evaluated with sharp constants.
\end{theorem}

\begin{proof}
Since \(\mu_D\le L_D\), we have \(e^{\mu_D(t-s)}\le e^{L_D(t-s)}\) for \(0\le s\le t\) and \(e^{\mu_D t}\le e^{L_D t}\). Multiplying these inequalities by the nonnegative quantities \(\|x_0-\hat x_\theta(0)\|\) and \(\delta_\theta(s)\), and then integrating, gives the result.
\end{proof}

\begin{remark}[Practical workflow]
\label{rem:training-workflow}
The recommended workflow is two-stage. First train the PINN with the standard residual and initial-condition losses. Then continue training with \eqref{eq:stage2-loss}, preferably using the one-sided upper constant when it is available. The final lower and upper certificates are then recomputed after training on a separate certification grid.
\end{remark}

\section{Numerical Examples}
\label{Sec 6}

\subsection{Nonlinear radial-growth plus rotation system}
\label{sec:num-radial-rotation}

We first consider the nonlinear two-dimensional system
\begin{equation}
 \dot x(t)=f(x(t)),\qquad x(0)=x_0,
\end{equation}
where
\begin{equation}
 f(x)=\bigl(1+\alpha\|x\|_2^2\bigr)x+\beta Jx,\qquad
 J=\begin{pmatrix}0&-1\\ 1&0\end{pmatrix},\qquad x_0=(0.2,0)^\top,
\end{equation}
with
\begin{equation}
 \alpha=0.4,\qquad \beta=6,\qquad T=1.
\end{equation}
This example is useful because the rotational term changes the Euclidean Lipschitz constant but cancels from the symmetric part of the Jacobian. In polar coordinates,
\begin{equation}
 r'(t)=\bigl(1+\alpha r(t)^2\bigr)r(t),\qquad \theta'(t)=\beta,
\end{equation}
and hence
\begin{equation}
 r(t)=\bigl((r_0^{-2}+\alpha)e^{-2t}-\alpha\bigr)^{-1/2},\qquad
 x(t)=r(t)\begin{pmatrix}\cos(\beta t)\\ \sin(\beta t)\end{pmatrix}.
\end{equation}
The Jacobian is
\begin{equation}
 Df(x)=\bigl(1+\alpha\|x\|_2^2\bigr)I+2\alpha xx^\top+\beta J,
\end{equation}
and therefore
\begin{equation}
 \sym(Df(x))=\bigl(1+\alpha\|x\|_2^2\bigr)I+2\alpha xx^\top.
\end{equation}
On the ball \(B_R=\{x:\|x\|_2\le R\}\), the constant certified coefficients are
\begin{equation}
 \ell_{B_R}=1,\qquad \mu_{B_R}=1+3\alpha R^2,\qquad L_{B_R}\le 1+3\alpha R^2+|\beta|.
\end{equation}
The skew-symmetric part \(\beta J\) is responsible for the large Lipschitz constant but does not affect \(\ell_{B_R}\) and \(\mu_{B_R}\).

\begin{figure}[htbp]
    \centering
    \includegraphics[width=0.98\textwidth]{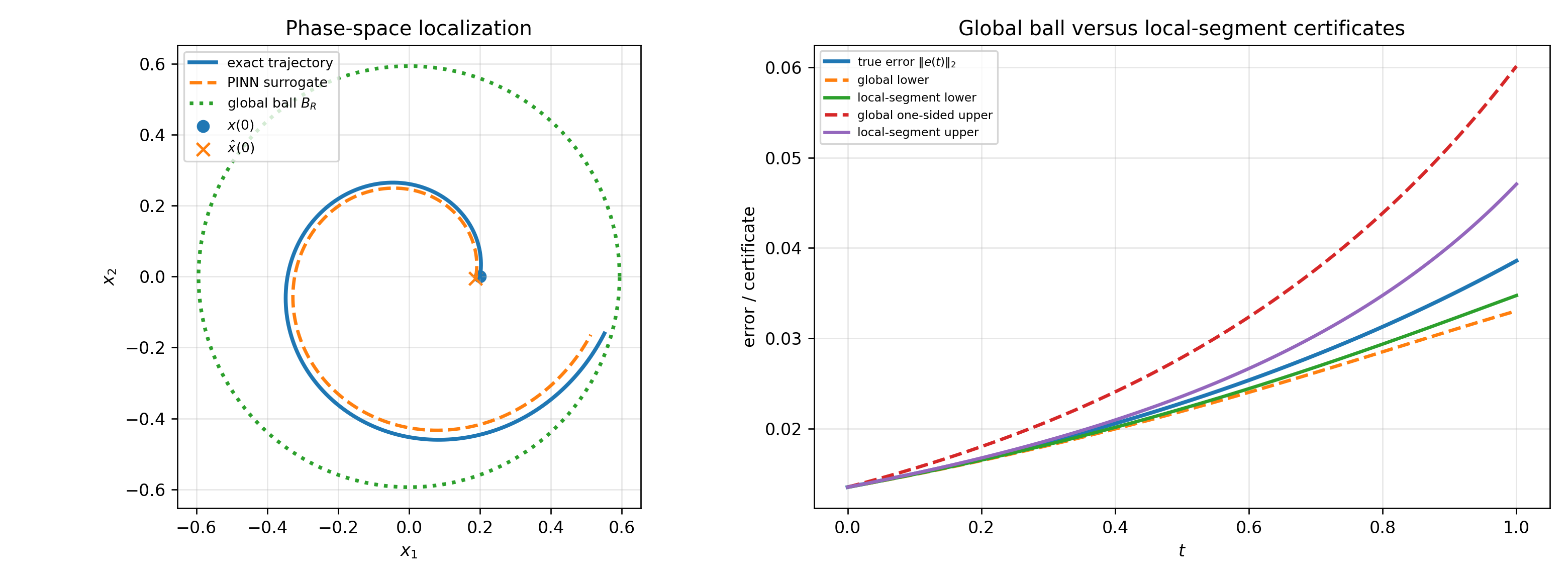}
    \caption{Nonlinear radial-growth plus rotation system. Left: exact and surrogate trajectories with the enclosing ball \(B_R\). Right: global ball-based lower and one-sided upper certificates compared with local-segment diagnostic curves and the true error. The local-segment curves are used only as a sharpness diagnostic in this synthetic example.}
    \label{fig:spiral-localization}
\end{figure}

\begin{figure}[htbp]
    \centering
    \includegraphics[width=0.98\textwidth]{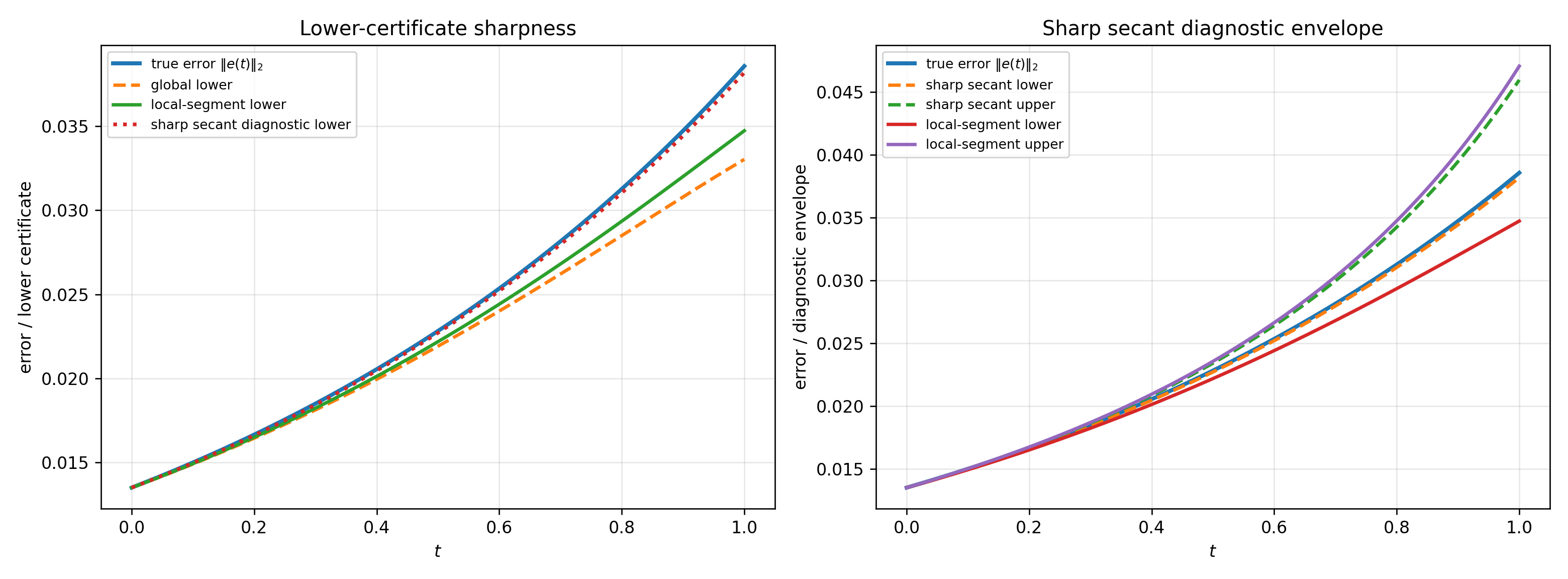}
    \caption{Diagnostic sharpness comparison for the nonlinear radial-growth plus rotation example. Left: lower-certificate sharpness over the full interval. Right: secant diagnostic envelope compared with the local-segment curves. The secant curves use the exact trajectory pair and are not used as black-box certificates.}
    \label{fig:spiral-diagnostics}
\end{figure}

\Cref{fig:spiral-localization} shows that both trajectories remain in the certified ball. At the final time, the true error is \(3.86\times10^{-2}\), the global-ball lower certificate is \(3.30\times10^{-2}\), the local-segment diagnostic lower curve is \(3.47\times10^{-2}\), the global one-sided upper certificate is \(6.01\times10^{-2}\), and the local-segment diagnostic upper curve is \(4.71\times10^{-2}\). The corresponding global Lipschitz upper certificate is valid but much larger, with final value about \(2.27\times10^{1}\), because it carries the rotational contribution \(|\beta|\). \Cref{fig:spiral-diagnostics} confirms that the segment and secant curves are useful for understanding sharpness, while the constant ball certificate remains the main computable certificate.

\begin{remark}[Constant examples and time-dependent extensions]
The reported certified constants in the main examples are constant on the certified domain. The local-segment and secant curves in \Cref{fig:spiral-localization,fig:spiral-diagnostics} are diagnostic comparisons for this synthetic problem. The general extension to time-dependent coefficients follows the integrating-factor remark in \Cref{Sec 2}.
\end{remark}

\FloatBarrier
\subsection{High-dimensional stiff linear ODE}
\label{sec:num-stiff}

The second example is a high-dimensional linear system, and it is also the place where the different linear constants are made explicit. We consider the stiff diagonal system in dimension \(d=96\):
\begin{equation}
 x_i'(t)=-\lambda_i x_i(t),
 \qquad
 \lambda_i=10^{4(i-1)/(d-1)},
 \qquad i=1,\ldots,d,
\label{eq:stiff-diagonal}
\end{equation}
with \(x_i(0)=(-1)^{i-1}/\sqrt d\) and \(T=1\). The stiffness ratio is \(10^4\). A full flow-map approximation on a box in \(\mathbb R^{96}\) would suffer from the curse of dimensionality, whereas the present trajectory-wise certificate only requires the trained path, residual evaluations, and matrix-vector operations.

For \(A=-\operatorname{diag}(\lambda_1,\ldots,\lambda_d)\),
\begin{equation}
 m(A)=-10^4,\qquad M(A)=-1,\qquad \|A\|_2=10^4.
\end{equation}
These values illustrate why positivity should not be imposed on one-sided constants. The system is dissipative, so the sharp one-sided constants are negative; this is valid for upper certification, but it means that no positive strong-monotonicity lower constant is available. To display all admissibility classes in one example, we also report non-sharp comparison constants
\begin{equation}
 \ell_{\rm adm}=-1.2\times10^4\le m(A),\quad
 \mu_{\rm adm}=0\ge M(A),\quad
 L_{\rm adm}=1.2\times10^4\ge\|A\|_2.
\end{equation}
Thus the six displayed numbers are two lower one-sided coefficients, two upper one-sided coefficients, and two Lipschitz constants; only \(m(A)\), \(M(A)\), and \(\|A\|_2\) are sharp.

The scalar lower certificate is not useful here because the system is dissipative and the hard-constrained ansatz has \(e(0)=0\). The signed-residual finite-probe certificate from \Cref{thm:signed-residual-probe} remains nontrivial. We use only the coordinate unit vectors \(P=\{e_1,\ldots,e_d\}\); hence the reported lower curve is \(\max_i |e_i(t)|\) reconstructed from the signed residual identity, not the full norm \(\|e(t)\|_2\).

\begin{figure}[htbp]
    \centering
    \includegraphics[width=0.85\textwidth]{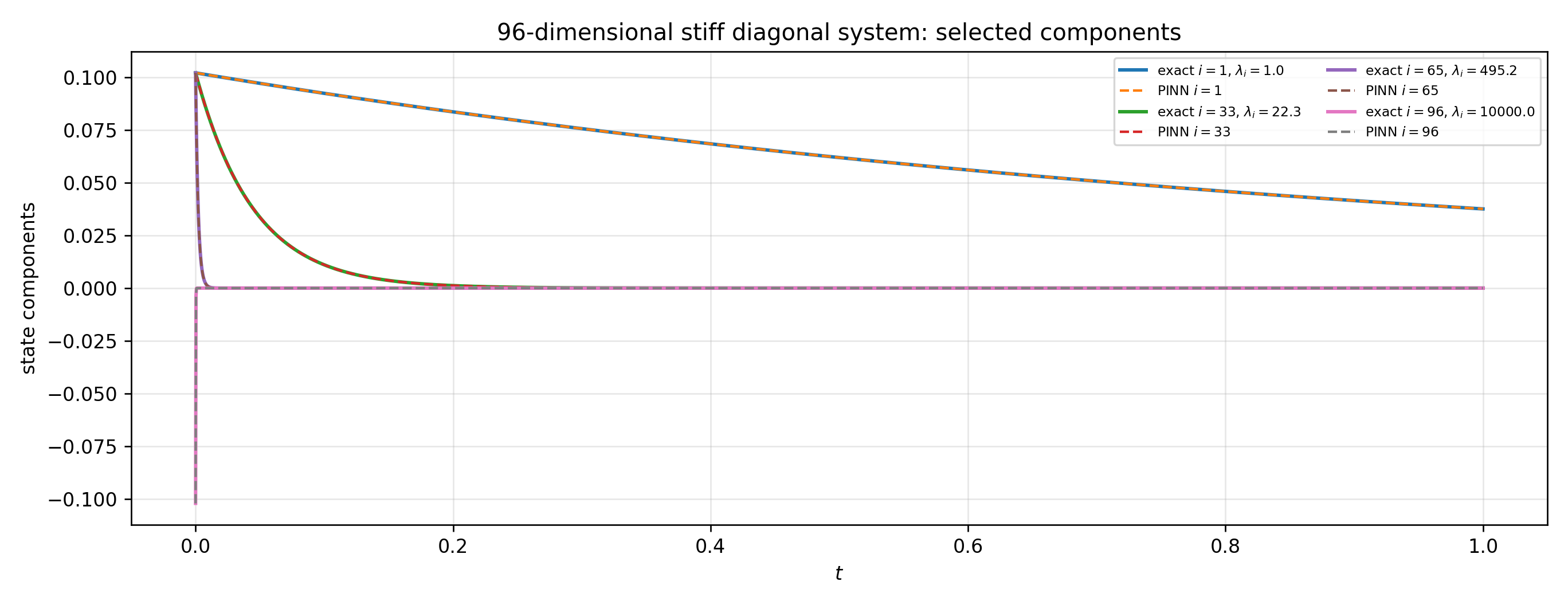}
    \caption{Selected components of the 96-dimensional stiff system. The surrogate resolves both slow and fast modes.}
    \label{fig:stiff-components}
\end{figure}

\begin{table}[htbp]
\centering
\begin{tabular}{lll}
\hline
\textbf{Class} & \textbf{Constant} & \textbf{Value}\\
\hline
One-sided (lower) & $m(A)$ & $-10^{4}$\\
One-sided (lower, admissible) & $\ell_{\rm adm}$ & $-1.2\times 10^{4}$\\
One-sided (upper) & $M(A)$ & $-1$\\
One-sided (upper, admissible) & $\mu_{\rm adm}$ & $0$\\
Lipschitz & $\|A\|_{2}$ & $10^{4}$\\
Lipschitz (admissible) & $L_{\rm adm}$ & $1.2\times 10^{4}$\\
\hline
\end{tabular}
\caption{Six constants for the stiff linear system. The one-sided constants may be negative; the Lipschitz constants are nonnegative.}
\label{tab:stiff-constants}
\end{table}

\begin{figure}[htbp]
    \centering
    \includegraphics[width=0.85\textwidth]{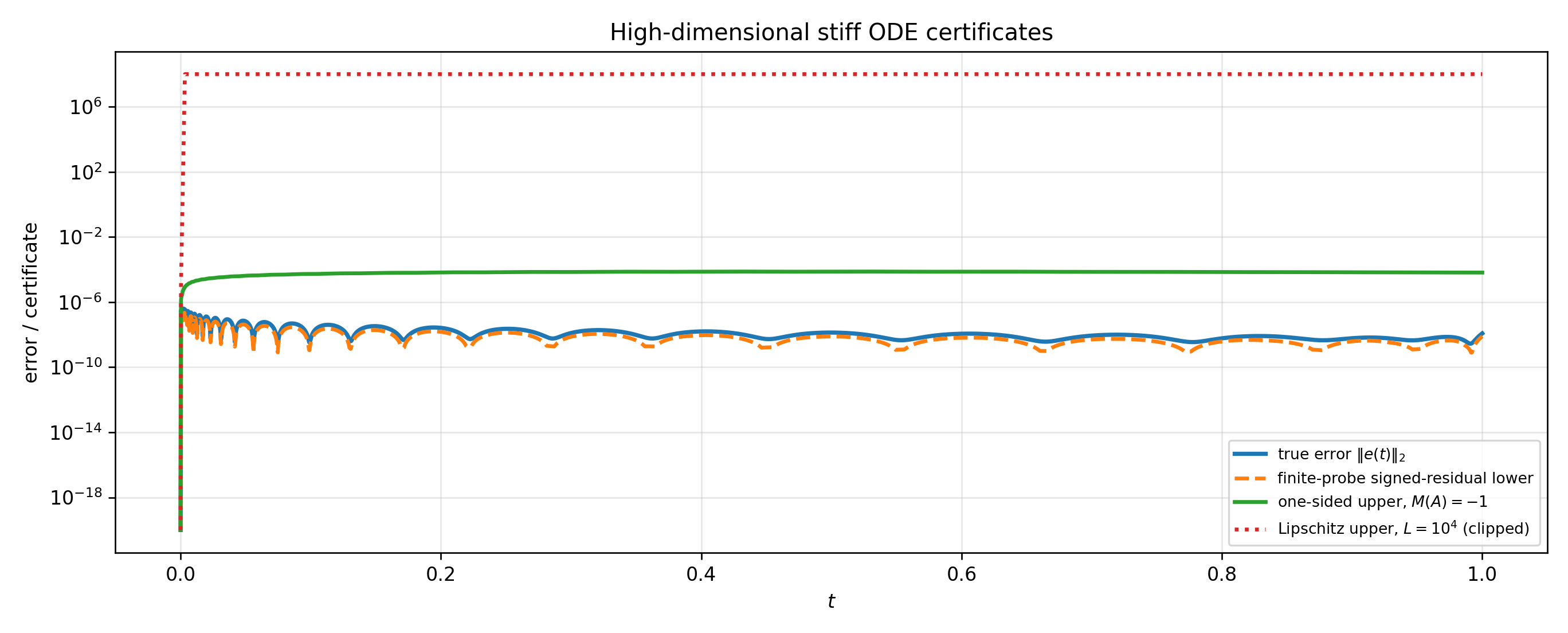}
    \caption{Certificates for the 96-dimensional stiff system. The coordinate-unit finite-probe signed-residual lower certificate is informative but does not coincide with the exact error norm. The one-sided upper certificate based on \(M(A)=-1\) remains finite, whereas the Lipschitz certificate based on \(\|A\|_2=10^4\) is clipped in the plot.}
    \label{fig:stiff-certificates}
\end{figure}

\Cref{fig:stiff-components,tab:stiff-constants,fig:stiff-certificates} show the effect of stiffness on certification. The maximum true error over the grid is \(4.01\times10^{-7}\). At the final time, the true error is \(1.19\times10^{-8}\), while the finite-probe lower certificate is \(7.94\times10^{-9}\), about \(66.5\%\) of the true error. The one-sided upper certificate is \(6.51\times10^{-5}\), whereas the Lipschitz upper certificate overflows the useful plotting scale. This example is the main use of the signed-residual finite-probe lower repair in high dimension.

\FloatBarrier
\subsection{Certificate-informed training on an unstable oscillatory system}
\label{sec:num-training}

The third example uses the certificate-informed training formulation from \Cref{Sec 5}. We consider
\begin{equation}
 x'(t)=Ax(t)+g(t),\qquad
 A=\begin{pmatrix}0.35&-20\\20&0.35\end{pmatrix},\qquad x(0)=0,
\label{eq:training-oscillation}
\end{equation}
with a smooth forcing \(g(t)\). The sharp upper one-sided constant and the Euclidean Lipschitz constant are
\begin{equation}
 M(A)=0.35,\qquad \|A\|_2=20.003.
\end{equation}
This is an unstable oscillatory system: the symmetric part gives the true amplitude growth, whereas the large skew-symmetric part makes the Lipschitz upper constant much larger. This is therefore a more informative test for upper-certificate training than a scalar ODE.

The PINN is hard-constrained at the initial condition. Stage one uses the standard residual loss. Stage two continues from the same checkpoint with three branches: ordinary PINN continuation, one-sided upper-certificate training based on \(U_{M(A)}\), and Lipschitz upper-certificate training based on \(U_{\|A\|_2}\). The auxiliary certificate state is recomputed from the current residual at each training epoch, as described in \Cref{Sec 5}. The lower certificate is not optimized; it is only reported after training as a finite-probe diagnostic.

\begin{figure}[htbp]
    \centering
    \includegraphics[width=0.95\textwidth]{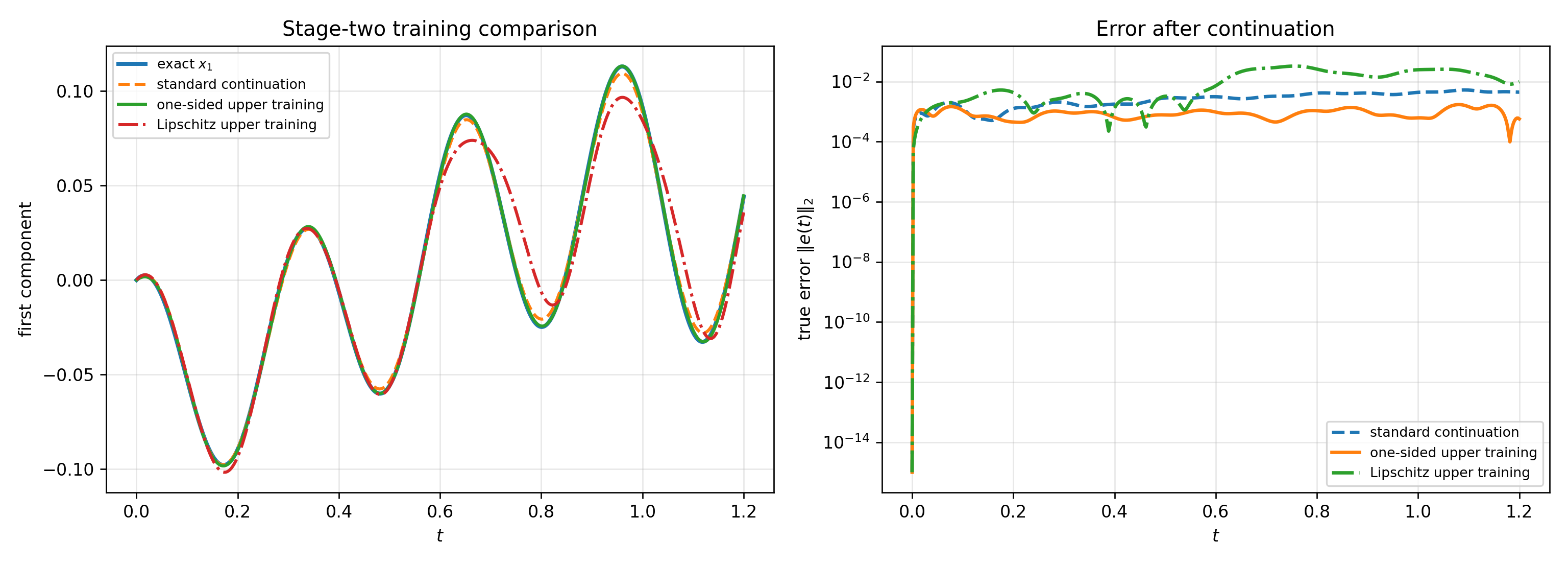}
    \caption{Certificate-informed training on the unstable oscillatory system. The one-sided upper-certificate branch tracks the exact solution more accurately and has the smallest true error after continuation.}
    \label{fig:training-solution-error}
\end{figure}

\begin{figure}[htbp]
    \centering
    \includegraphics[width=0.95\textwidth]{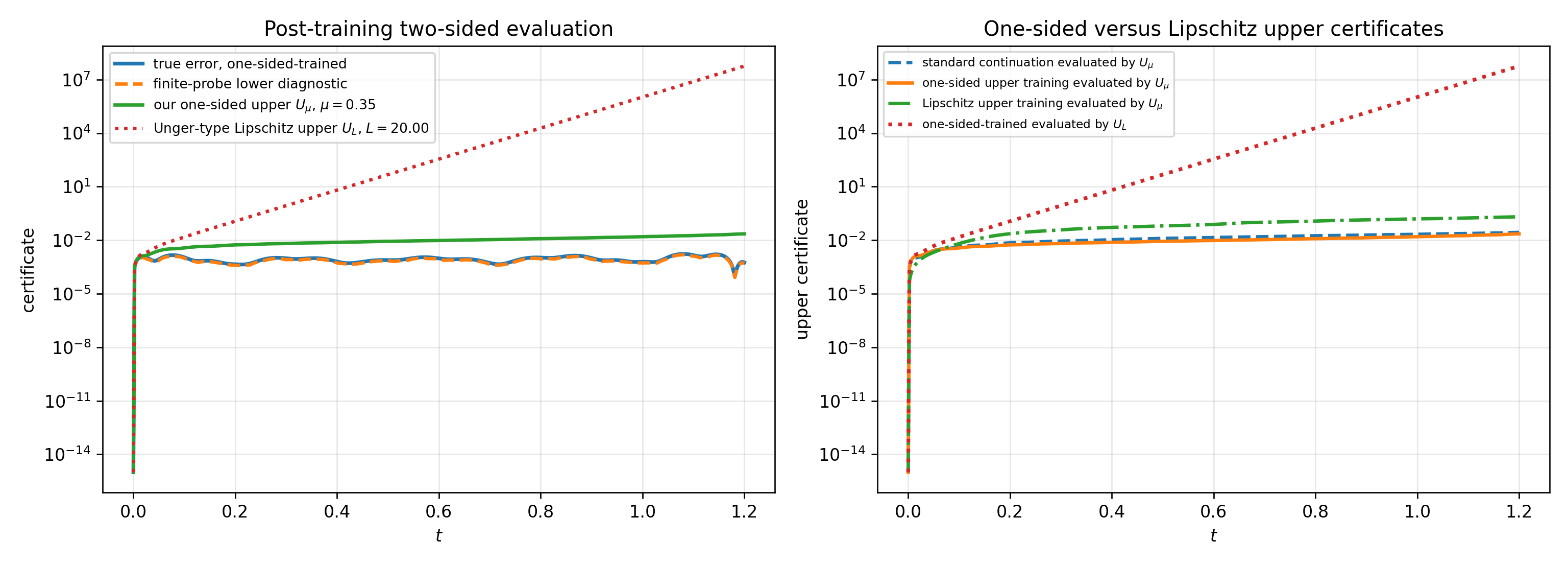}
    \caption{Post-training certificate evaluation for the unstable oscillatory system. Left: true error, finite-probe lower diagnostic, one-sided upper certificate, and Lipschitz-based upper certificate for the one-sided trained model. Right: comparison of propagated upper certificates across the three branches.}
    \label{fig:training-certificates}
\end{figure}

\begin{figure}[htbp]
    \centering
    \includegraphics[width=0.72\textwidth]{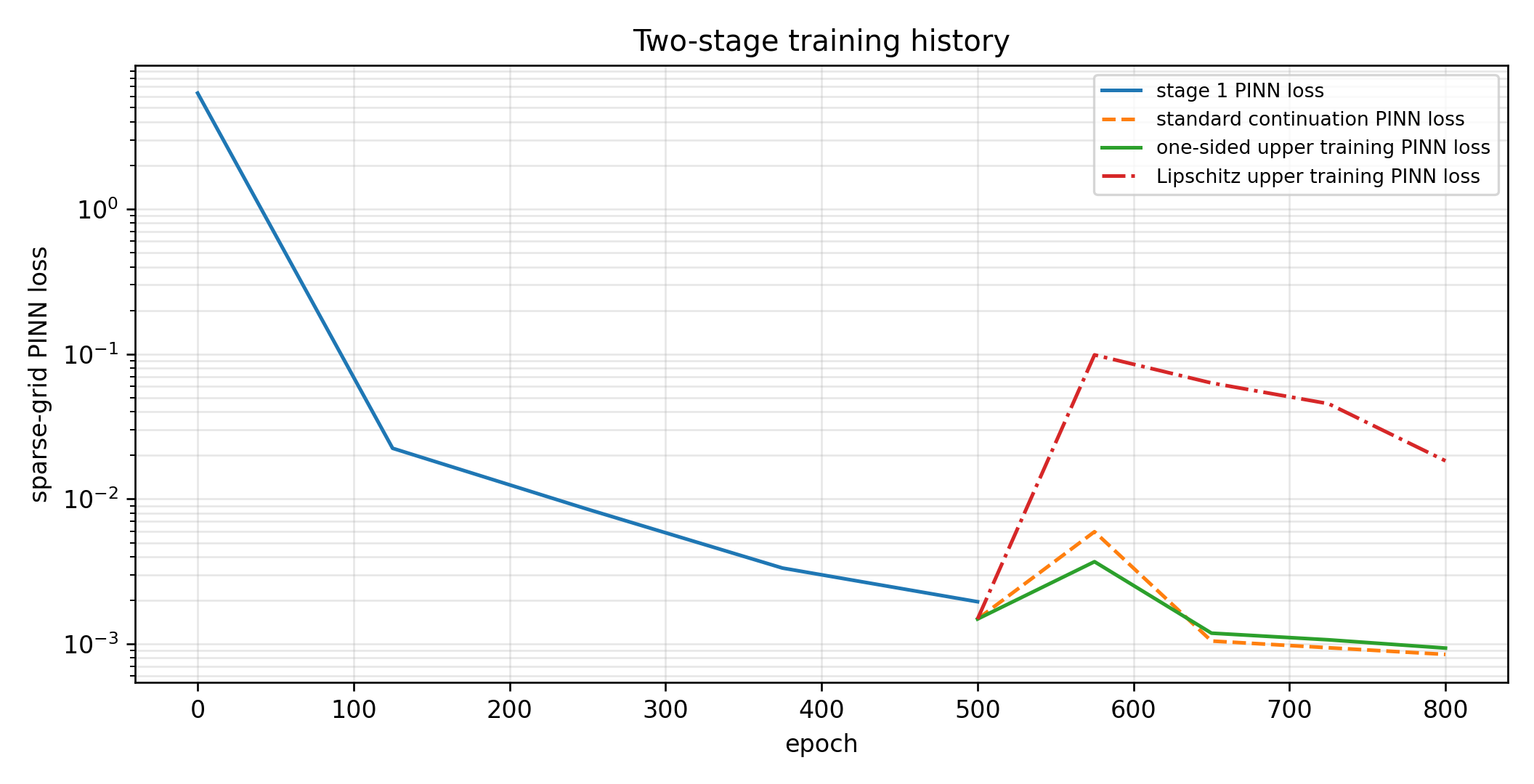}
    \caption{Two-stage training history. Stage one uses the standard PINN loss. Stage two compares ordinary continuation with the two upper-certificate regularizers.}
    \label{fig:training-history}
\end{figure}

\Cref{fig:training-solution-error,fig:training-certificates,fig:training-history} show that the one-sided upper-certificate branch is the most effective of the three continuations. Its final error is \(5.68\times10^{-4}\), compared with \(4.43\times10^{-3}\) for the standard continuation and \(9.79\times10^{-3}\) for the Lipschitz-refined branch. For the one-sided trained model, the final one-sided upper certificate is \(2.26\times10^{-2}\), whereas the final Lipschitz-based upper certificate is about \(5.70\times10^7\). This supports the practical recommendation to use the one-sided upper certificate for training whenever a reliable one-sided constant is available.

\FloatBarrier
\section{Discussion and Conclusion}
\label{Sec 7}

We developed a computable two-sided \emph{a posteriori} certification framework for PINN approximations of ordinary differential equations. In contrast to certification results based only on global Lipschitz continuity, which provide upper error bounds, the present approach gives both lower and upper certificates under local one-sided growth assumptions. In the nonlinear setting, the bounds are governed by admissible lower and upper one-sided constants on a certified state-space domain. In the linear autonomous case, these constants reduce to the sharp symmetric-part quantities
\begin{equation}
m(A)=\lambda_{\min}\!\left(\frac{A+A^\top}{2}\right),
\quad
M(A)=\lambda_{\max}\!\left(\frac{A+A^\top}{2}\right).
\end{equation}

The simultaneous availability of lower and upper bounds gives a geometric localization of the exact solution around the PINN prediction. Namely, if
\begin{equation}
\underline E(t)\le \|x(t)-\hat x(t)\|\le \overline E(t),
\end{equation}
then the exact state belongs to
\begin{equation}
\mathcal X_{\mathrm{cert}}(t)
=
\left\{
z\in\mathbb{R}^n:
\underline E(t)\le \|z-\hat x(t)\|\le \overline E(t)
\right\}.
\end{equation}
Thus the certificate defines a union of intervals in one dimension, an annulus in two dimensions, a spherical shell in three dimensions, and a closed annular region in general dimension. The band width $\overline E(t)-\underline E(t)$ therefore quantifies the sharpness of the localization.

The numerical examples illustrate these conclusions in three regimes. The nonlinear radial-growth plus rotation example shows that localization and one-sided constants are informative in a genuinely nonlinear multidimensional system; the rotational term enlarges the Lipschitz constant without affecting the symmetric-part growth. The high-dimensional stiff linear example displays the six admissible linear constants and shows that a finite-probe signed-residual lower certificate can remain informative when the scalar lower certificate is zero. The unstable oscillatory training example shows how the upper certificate can be used as a refinement term, with the lower certificate retained as a post-training diagnostic.

\subsection*{Limitations and admissible model classes}

The framework is not intended to be universal. Its usefulness depends on the availability of suitable local one-sided growth information and a certified domain containing the exact and PINN trajectories. The scalar lower certificate is most informative when a positive lower one-sided coefficient \(\ell_D>0\) can be verified and when the transported initial mismatch dominates the residual accumulation. This condition is natural in local expansive or positive-feedback regimes, for example in early-growth epidemic linearizations, low-incidence cooperative models, and autocatalytic reaction networks on domains where the symmetric part of the Jacobian has a positive lower bound. Dissipative systems, including the stiff linear example, typically have \(\mu_D<0\) and may have \(\ell_D<0\); then the upper one-sided certificate can be very useful, but the scalar lower certificate should not be expected to be positive.

If the scalar lower condition fails, a zero scalar lower value is unavoidable from scalar residual-majorant information alone. The only repair mechanism used in the numerical section is the signed-residual finite-probe certificate for linear inhomogeneous systems. For nonlinear hard-constrained PINNs, nontrivial lower repair would require additional assumptions or observables, such as adjoint-remainder bounds or trusted interior data, and is left for future work. Conservative or symplectic systems, Hamiltonian dynamics, and non-monotone interactions such as classical Lotka--Volterra models may therefore require problem-specific lower-certificate mechanisms \cite{marsdenratiu1999,nutku1990hamiltonian}. Finally, conservative estimates of the local constants, the residual majorant, or the RK4 remainder constants widen the certified band.

Overall, the results show that PINN solutions of ODEs can be equipped with rigorous, computable error enclosures rather than only point predictions. Future work will focus on sharper local constants, less conservative residual majorants, and extensions to broader nonlinear systems and PDE-based PINNs.

\bibliographystyle{elsarticle-num}
\bibliography{references}

\end{document}